\documentclass{article}

\usepackage{microtype}
\usepackage{graphicx}
\usepackage{subcaption}
\usepackage{booktabs} 

\usepackage{hyperref}

\usepackage{xcolor}  

\usepackage[preprint]{icml2026}

\usepackage{amsmath}
\usepackage{amssymb}
\usepackage{mathtools}
\usepackage{amsthm}

\usepackage{booktabs}
\usepackage{multirow} 
\usepackage{graphicx}

\usepackage[capitalize,noabbrev]{cleveref}

\theoremstyle{plain}
\newtheorem{theorem}{Theorem}[section]

\newtheorem{lemma}[theorem]{Lemma}

\theoremstyle{definition}

\newtheorem{assumption}[theorem]{Assumption}
\theoremstyle{remark}
\newtheorem{remark}[theorem]{Remark}

\usepackage[textsize=tiny]{todonotes}

\icmltitlerunning{Submission and Formatting Instructions for ICML 2026}

\begin{document}
\pagestyle{plain} 
\thispagestyle{plain} 

\twocolumn[
  \icmltitle{State Rank Dynamics in Linear Attention LLMs}



  \icmlsetsymbol{equal}{*}

    \begin{icmlauthorlist}
    \icmlauthor{Ao Sun}{mt,cuhk,intern} 
    \icmlauthor{Hongtao Zhang}{ucas,cas}
    \icmlauthor{Heng Zhou}{ustc}
    \icmlauthor{Yixuan Ma}{hit}
    \icmlauthor{Yiran Qin}{ox}
    \icmlauthor{Tongrui Su}{mt,intern,ucas}
    \icmlauthor{Yan Liu}{mt,intern,thu}
    \icmlauthor{Zhanyu Ma}{mt}
    \icmlauthor{Jun Xu}{mt}
    \icmlauthor{Jiuchong Gao}{mt}
    \icmlauthor{Jinghua Hao}{mt}
    \icmlauthor{Renqing He}{mt}
    \end{icmlauthorlist}
    
    \icmlaffiliation{mt}{Meituan}
    \icmlaffiliation{cuhk}{The Chinese University of Hong Kong, Shenzhen}
    \icmlaffiliation{ucas}{University of Chinese Academy of Sciences}
    \icmlaffiliation{cas}{CAS Key Laboratory of AI Safety, Institute of Computing Technology, CAS}
    \icmlaffiliation{ustc}{University of Science and Technology of China}
    \icmlaffiliation{hit}{Harbin Institute of Technology}
    \icmlaffiliation{ox}{University of Oxford}
    \icmlaffiliation{thu}{Tsinghua University}
    \icmlaffiliation{intern}{Work done during internship at Meituan.} 

    \icmlcorrespondingauthor{Jun Xu}{xujun58@meituan.com}
    \icmlcorrespondingauthor{Jiuchong Gao}{gaojiuchong@meituan.com}

  \icmlkeywords{Machine Learning, ICML}

  \vskip 0.3in
]



\printAffiliationsAndNotice{}  

\begin{abstract}
Linear Attention Large Language Models (LLMs) offer a compelling recurrent formulation that compresses context into a fixed-size state matrix, enabling constant-time inference. However, the internal dynamics of this compressed state remain largely opaque. In this work, we present a comprehensive study on the runtime state dynamics of state-of-the-art Linear Attention models. We uncover a fundamental phenomenon termed State Rank Stratification, characterized by a distinct spectral bifurcation among linear attention heads: while one group maintains an effective rank oscillating near zero, the other exhibits rapid growth that converges to an upper bound. Extensive experiments across diverse inference contexts reveal that these dynamics remain strikingly consistent, indicating that the identity of a head,whether low-rank or high-rank,is an intrinsic structural property acquired during pre-training, rather than a transient state dependent on the input data. Furthermore, our diagnostic probes reveal a surprising functional divergence: low-rank heads are indispensable for model reasoning, whereas high-rank heads exhibit significant redundancy.  Leveraging this insight, we propose Joint Rank-Norm Pruning, a zero-shot strategy that achieves a 38.9\% reduction in KV-cache overhead while largely maintaining model accuracy. 
\end{abstract}



\begin{figure}
    \centering
    \includegraphics[width=1.0\linewidth]{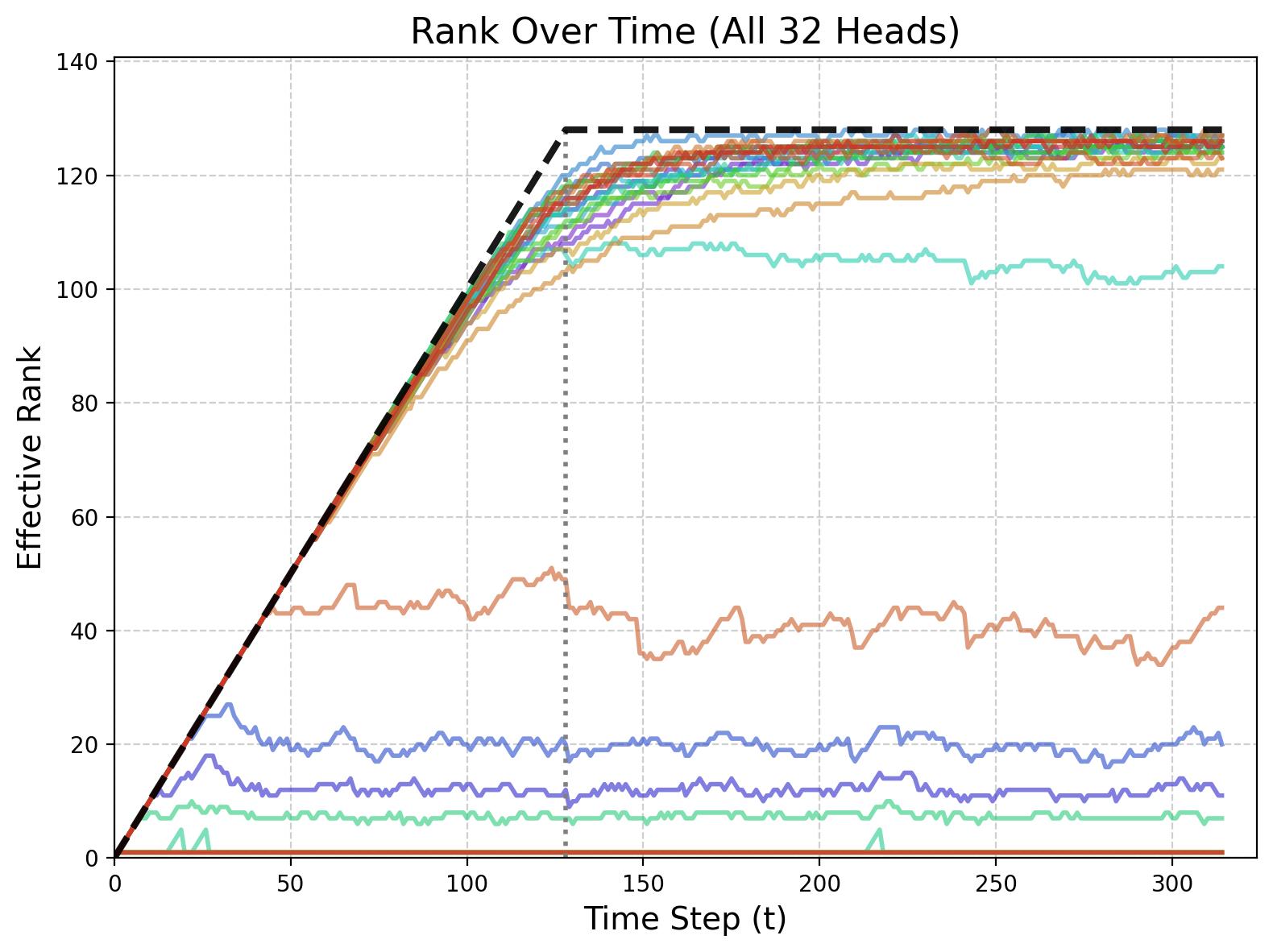}
    \caption{\textbf{State Rank Dynamics across 32 attention heads in a representative layer.} The dashed black line indicates the theoretical rank upper bound. The plot vividly illustrates State Rank Stratification, revealing a clear bifurcation between distinct high-rank and low-rank regimes.}
    \label{fig:layer_0_32heads_rank}
\end{figure}

\section{Introduction}
\label{sec:intro}

The Transformer\citep{vaswani2017attention} architecture has fundamentally revolutionized natural language processing, yet its standard self-attention mechanism suffers from quadratic computational complexity, creating a significant bottleneck for long-context applications. To reconcile high performance with efficiency, Linear Attention Large Language Models (LLMs)---exemplified by the state-of-the-art Qwen3-Next\citep{qwen3technicalreport} series (Instruct and Thinking)---have emerged as a compelling solution. By recurrently compressing unbounded history into a fixed-size state matrix $\mathbf{S}(t) \in \mathbb{R}^{d \times d}$, these models achieve linear scaling during training and constant memory usage during inference. However, despite its efficiency, the internal dynamics of this compressed state remain largely opaque, effectively operating as a black box where the evolution of information storage and retention is poorly understood.


To probe the internal mechanisms of the model, we conducted a systematic analysis of the state dynamics of attention heads within the linear attention layers of Qwen-Next. Our investigation yields a fundamental empirical observation, theoretically supported by mathematical induction: the rank dynamics of the state matrix are strictly constrained by an upper bound. Within this bound, we identify a striking Rank Stratification phenomenon. As illustrated in Figure~\ref{fig:layer_0_32heads_rank}, attention heads naturally bifurcate into two distinct regimes: \textit{High-Rank Heads}, which tightly track the theoretical upper bound to maximize capacity, and \textit{Low-Rank Heads}, which consistently operate in a low-dimensional manifold. 


Crucially, we demonstrate that this stratification exhibits Temporal Invariance. By calculating the cosine similarity between the rank vectors of a layer at time $t$ and $t+\Delta t$, we confirm that a head's role (high vs. low rank) remains stable throughout the generation process. Furthermore, to verify the Data Invariance of this distribution, we curate RankViz, a comprehensive dataset encompassing Mathematical Reasoning, STEM, Long-Context, and adversarial attack samples. Our analysis on RankViz reveals that, with the exception of specific attack patterns, the distribution of high and low-rank heads remains remarkably consistent across diverse domains.

Leveraging these insights, we propose Joint Rank-Norm Pruning, a novel, data-free inference acceleration strategy. Contrary to the intuition that higher rank implies higher importance, our diagnostic probes reveal that saturated High-Rank heads often exhibit significant redundancy. By exploiting the discovered stratification, we permanently prune these redundant high-rank heads without the need for retraining or auxiliary predictors. Extensive experiments demonstrate that our method significantly reduces computational overhead while maintaining the model's generative performance.

Our main contributions are summarized as follows:

\begin{enumerate}
    \item \textbf{Identification of Rank Stratification via Spectral Analysis:} We conduct the first systematic spectral analysis of the recurrent state matrix $\mathbf{S}_t$ in Linear Attention LLMs. Through this process, we identify a strict rank upper bound and the phenomenon of \textit{Temporal Invariance}, revealing that attention heads bifurcate into stable high-rank and low-rank regimes across diverse datasets.
    
    \item \textbf{Theoretical Analysis of State Dynamics:} Theoretical Foundation of Cache Dynamics. We establish a rigorous theoretical framework to elucidate the emergence of stable patterns in Linear Attention Transformers. By analyzing the recursive dynamics of rank-one updates, we derive two complementary theorems: we prove that the effective rank vector stabilizes via \textbf{\textit{Rank Saturation}} (as heads reach their intrinsic subspace capacities), and the nuclear norm vector maintains directional invariance via \textbf{\textit{Norm Accumulation}} (driven by the accumulative inertia of the growing cache, where the relative impact of new tokens diminishes). These results formally justify why attention heads exhibit strictly preserved characteristics throughout the inference process.
    
    

    \item \textbf{Efficient Inference via Mechanism-Aware Pruning:} We uncover the spectral root cause of Long-Context Collapse, demonstrating that preserving specific low-rank dynamics is essential for maintaining retrieval capabilities. Guided by this diagnostic insight, we propose Joint Rank-Norm Pruning, a training-free compression strategy. This approach effectively converts redundancy into efficiency, achieving a 38.9\% reduction in KV-cache memory usage while maintaining model accuracy comparable to the original baseline.
\end{enumerate}


\section{Preliminary}

\subsection{Standard Softmax Attention}
The foundation of modern Transformers is the scaled dot-product attention mechanism\citep{vaswani2017attention}. Given input projections for queries $\mathbf{Q}$, keys $\mathbf{K}$, and values $\mathbf{V}$ (all in $\mathbb{R}^{n \times d}$), the attention output is computed as:
\begin{equation}
    \text{Attn}(\mathbf{Q}, \mathbf{K}, \mathbf{V}) = \text{softmax}\left(\frac{\mathbf{Q}\mathbf{K}^\top}{\sqrt{d}} + \mathbf{M}\right)\mathbf{V},
    \label{eq:softmax_attn}
\end{equation}
where $\mathbf{M}$ denotes the additive causal mask (with $-\infty$ in the upper triangle) to prevent future information leakage during autoregressive generation.

\subsection{Linear Attention}
Linear attention aims to reduce computational complexity to linear time via kernel tricks or a recurrent view. From an \textit{online learning} perspective, these models maintain a recurrent state matrix $\mathbf{S}(t) \in \mathbb{R}^{d \times d}$ to store historical key--value associations. Depending on the update dynamics, we categorize them into two paradigms:

\paragraph{Standard Linear Attention: Correlation Accumulation.}
The most fundamental form treats state updates as simple Hebbian accumulation\citep{katharopoulos2020transformers}. Its implicit optimization objective is to maximize the alignment between the state and observed data, which corresponds to minimizing the following negative correlation loss:
\begin{equation}
    \mathcal{L}_t^{\text{Linear}}(\mathbf{S}) = -\langle \mathbf{S}^\top \mathbf{k}_t, \mathbf{v}_t \rangle.
    \label{eq:linear_loss}
\end{equation}
This yields the standard additive update rule:
\begin{equation}
    \mathbf{S}(t) = \mathbf{S}(t-1) + \mathbf{k}_t \mathbf{v}_t^\top, \quad \mathbf{o}_t = \mathbf{S}(t)^\top \mathbf{q}_t.
    \label{eq:linear_update}
\end{equation}
The absence of a forgetting mechanism in this formulation leads to unbounded state growth, introducing interference over long contexts.

\paragraph{DeltaNet: Reconstruction-based Correction.}
To mitigate unbounded growth, DeltaNet\citep{yang2024parallelizing,yang2024gated} reframes the state update as an online regression task. The goal is for the state matrix to accurately reconstruct the value $\mathbf{v}_t$ given the key $\mathbf{k}_t$. Consequently, the optimization objective shifts to minimizing the instantaneous reconstruction error:
\begin{equation}
    \mathcal{L}_t^{\text{Delta}}(\mathbf{S}) = \frac{1}{2} \| \mathbf{S}^\top \mathbf{k}_t - \mathbf{v}_t \|^2.
    \label{eq:delta_loss}
\end{equation}
Applying a single gradient descent step with learning rate $\beta_t$ on this loss yields the Delta Rule. This introduces an error-correction term, allowing the model to dynamically adjust its memory:
\begin{equation}
\begin{aligned}
    \mathbf{S}(t) &= \mathbf{S}(t-1) - \beta_t \nabla_{\mathbf{S}}\mathcal{L}_t \\
    &= (\mathbf{I} - \beta_t \mathbf{k}_t \mathbf{k}_t^\top)\mathbf{S}(t-1) + \beta_t \mathbf{k}_t \mathbf{v}_t^\top.
\end{aligned}
\label{eq:delta_update}
\end{equation}
This update mechanism endows the linear attention model with the capability to "forget" and "correct," representing a more robust rank-1 update strategy.




\subsection{Experimental Metric}
\paragraph{Effective Rank}
Directly computing the algebraic rank of the state matrix $\mathbf{S}(t) \in \mathbb{R}^{d \times d}$ is numerically unstable due to floating-point noise. To robustly quantify information content across varying energy scales, we adopt a relative threshold-based Effective Rank metric. Specifically, given the singular values sorted in descending order $\sigma_1 \ge \dots \ge \sigma_d \ge 0$, we define the effective rank as the count of singular values exceeding a fraction $\epsilon$ of the spectral norm (the largest singular value):
\begin{equation}
    \text{Rank}_{\text{eff}}(\mathbf{S}(t)) = \sum_{i=1}^d \mathbb{I}(\sigma_i > \epsilon \cdot \sigma_1),
\end{equation}
where $\mathbb{I}(\cdot)$ is the indicator function and $\sigma_1 = \|\mathbf{S}(t)\|_2$. In our experiments, we set the relative tolerance $\epsilon = 10^{-4}$. This strategy aligns with standard conventions in numerical linear algebra libraries (e.g., PyTorch, NumPy), ensuring that the rank estimation remains consistent independent of the absolute magnitude of the state matrix.

\paragraph{Cosine Similarity.}
To quantify the stability of head-specific characteristics across inference steps, we employ \textit{Cosine Similarity}. Unlike Euclidean distance, which is sensitive to the magnitude growth inherent in the KV cache accumulation, cosine similarity isolates the directional alignment of the feature vectors (e.g., the rank vector $\mathbf{r}(t)$ or norm vector $\mathbf{n}(t)$). Formally, for two non-zero vectors $\mathbf{u}, \mathbf{v} \in \mathbb{R}^H$, it is defined as:
\begin{equation}
\text{CosSim}(\mathbf{u}, \mathbf{v}) 
\;=\; 
\frac{\mathbf{u}^\top \mathbf{v}}{\|\mathbf{u}\|_2 \|\mathbf{v}\|_2}.
\label{eq:metric_cossim}
\end{equation}
A value close to $1$ implies that the relative distribution of the metric across attention heads remains invariant, regardless of the global scale increase.


\section{Identification of State Rank Dynamics}
\label{sec:stratification}

In this section, we systematically investigate the runtime information dynamics of Linear Attention LLMs. By tracking the spectral properties of the state matrix $\mathbf{S}(t)$ in Qwen-Next, we uncover a set of governing laws that reveal a fundamental functional bifurcation among attention heads.


We visualize the evolution of $\text{Rank}_{\text{eff}}(\mathbf{S}(t))$ (Figure~\ref{fig:layer_0_32heads_rank}). Our analysis proceeds from the basic physical constraints to the emergent structural properties, summarizing three key observations.

\subsection{The State Rank Upper Bound}
\label{sec:law_bound}

To understand the memory limitations of the state matrix $\mathbf{S}(t)$, we analyze its algebraic rank properties. We establish a formal proposition demonstrating that the rank of the memory state in both Standard Linear Attention and DeltaNet is inherently bounded by the sequence length and feature dimension.

\begin{theorem}[State Rank Upper Bound]\label{thm1}
Assuming the initial state $\mathbf{S}_0 = \mathbf{0}$, for both Standard Linear Attention and DeltaNet, the rank of the state matrix $\mathbf{S}(t) \in \mathbb{R}^{d \times d}$ at any time step $t$ satisfies:
\begin{equation}
    \text{rank}(\mathbf{S}(t)) \le \min(t, d)
\end{equation}
\end{theorem}

The proof is provided in Appendix \ref{sec:rank_upper_bound}.
\textbf{The Capacity Bottleneck.} The inequality $\text{rank}(\mathbf{S}(t)) \le \min(t, d)$ highlights a critical bottleneck in linear attention mechanisms. In the early phase where $t < d$, the model operates in a \textit{Rank Growth} regime, theoretically capable of memorizing discrete historical tokens perfectly as linearly independent updates. However, once $t \ge d$, the state matrix enters a \textit{Rank Saturation} regime. In this state, any new information $\mathbf{v}_{t+1} \mathbf{k}_{t+1}^\top$ must inevitably lie within the existing $d$-dimensional subspace, forcing a trade-off between overwriting historical memories and ignoring new inputs.

\begin{figure*}[t!]
    \centering
    \includegraphics[width=1.0\linewidth]{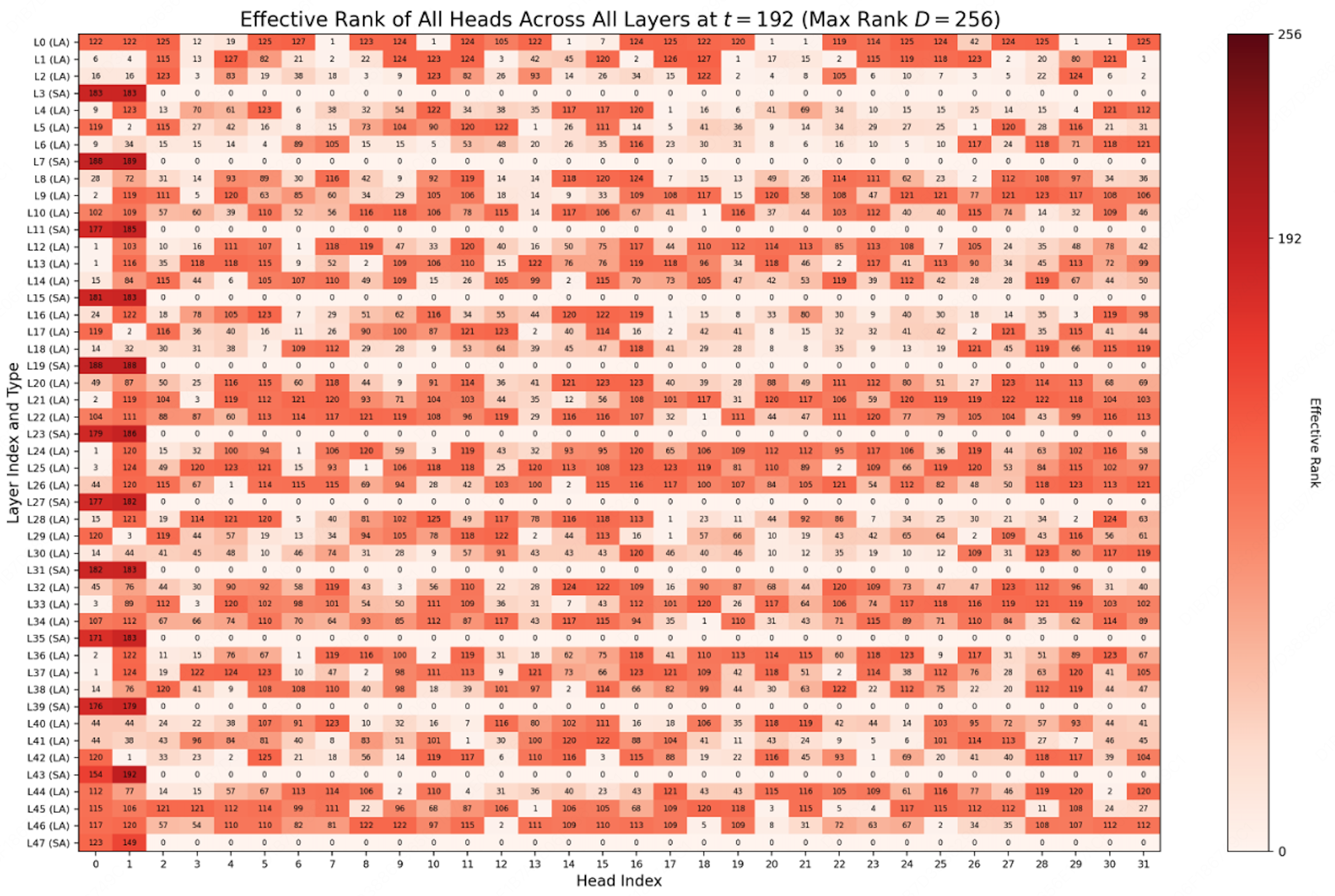}
    \caption{\textbf{Visualization of the Effective Rank for all 48 layers in Qwen3-Next.} Darker colors indicate a higher effective rank. The visualization reveals that the number of high-rank and low-rank heads is approximately equal within the linear attention layers, with both types appearing to be randomly distributed across all 48 layers. LA: Linear Attention layer; SA: Softmax Attention layer.}
    \label{fig:cross_all_layers}
\end{figure*}

\subsection{Temporal Consistency}
\label{sec:law_consistency}

A critical question arises regarding the nature of state stratification: Is a head's classification as low-rank or redundant a transient state dependent on the specific context of the current window, or is it an intrinsic, persistent trait of the head itself? To investigate this, we analyze the stability of head dynamics over widely separated time steps (e.g., comparing $t=128$ with $t=2048$) across the entire generation process. We formalize this stability through two distinct definitions of consistency:

\paragraph{Rank Consistency.}
This metric measures whether the relative complexity of information handled by different heads remains stable. We define the rank vector at step $t$ as $\mathbf{r}_t = [\mathrm{Rank}_{\text{eff}}^{(1)}(t), \dots, \mathrm{Rank}_{\text{eff}}^{(H)}(t)] \in \mathbb{R}^H$. We compute the Spearman rank correlation coefficient between rank vectors at different time steps, denoted as $\rho(\mathbf{r}_{t_1}, \mathbf{r}_{t_2})$.
Our empirical results demonstrate remarkably high correlations ($R > 0.90$) across most layers (Figure\ref{fig:temporal_consistency}). This indicates a strict Order Preservation: if Head A has a higher effective rank than Head B at the early stage of inference, it is statistically guaranteed to maintain this higher complexity throughout the sequence. This implies that the Redundant Heads are structurally destined to be low-rank, regardless of the input tokens.

\paragraph{Norm Consistency.}
While rank reflects the structural dimension of the state matrix, the nuclear norm $\|\mathbf{S}(t)\|_*$ reflects its energy or information magnitude. We similarly define the norm vector $\mathbf{n}_t = [\|\mathbf{S}(t)^{(1)}\|_*, \dots, \|\mathbf{S}(t)^{(H)}\|_*] \in \mathbb{R}^H$ and analyze the cosine similarity between norm vectors over time.
We observe that this similarity remains close to $1.0$ (typically $>0.98$) (Figure \ref{fig:temporal_consistency}). This Norm Consistency suggests that the relative signal strength of each head is locked in. Heads that accumulate high energy do so consistently, while those with minimal contribution do not suddenly surge in importance later in the sequence.

\noindent\textbf{Conclusion.} The confluence of Rank and Norm Consistency reveals a fundamental property of Linear Attention models: redundancy is a persistent trait, not a transient state. This stability is pivotal for efficiency, as it suggests that redundant heads can be reliably identified and pruned at the very early stages of inference without the need for continuous monitoring or re-evaluation.

\subsection{Data Independence}
\label{sec:law_independence}

Finally, we examine whether this stratification is sensitive to the input domain. Theoretically, the state matrix $\mathbf{S}(t)$ is a coupled function of the input sequence $\mathcal{X} = \{x_1, \dots, x_t\}$ and the model parameters $\theta = \{\mathbf{W}_K, \mathbf{W}_V\}$:
\begin{equation}
    \mathbf{S}(t)(\mathcal{X}; \theta) = \sum_{i=1}^t (\mathbf{W}_K x_i)(\mathbf{W}_V x_i)^\top.
    \label{eq:state_dependency}
\end{equation}
Given Equation \eqref{eq:state_dependency}, distinct data domains (e.g., natural language versus programming code) possess divergent feature distributions, which should theoretically drive the state matrix along different evolutionary trajectories. To investigate this, we replicate our spectral analysis across diverse datasets, including WikiText (natural language), GitHub (code), and arXiv (mathematical texts).

Strikingly, while the specific numerical entries of $\mathbf{S}(t)$ fluctuate significantly with the data distribution, the structural distribution of high-rank versus low-rank heads remains strictly invariant. The same set of heads consistently falls into the redundan cluster regardless of the input modality. This Data Independence indicates that Rank Stratification is not a dynamic reaction to simple versus complex inputs, but an intrinsic property encoded within the singular value spectra of the pre-trained weight matrices $\mathbf{W}_K$ and $\mathbf{W}_V$. This insight serves as the cornerstone of our proposed data-free pruning method, as it eliminates the need for calibration on domain-specific validation sets.

\subsection{Diagnostics with Synthetic Data Distributions}
\label{sec:special_data}
We examine the state dynamics under adversarial looping attacks, involving rare characters (Scenario I, Figure~\ref{fig:luanma1}) and common numbers (Scenario II, Figure~\ref{fig:luanma2}). 
In both cases, we observe significant deviations from the standard stratification pattern, characterized primarily by periodicity and rank reduction. 
Instead of converging to the theoretical upper bound, the effective rank in both scenarios exhibits cyclic oscillations mirroring the input repetition and suffers from a substantial contraction in magnitude. 
This suggests that regardless of the token type, repetitive input patterns prevent the state matrix from effectively accumulating information, forcing the attention heads into a constrained, oscillatory low-rank regime.

\section{Theoretical Analysis of Rank and Norm Preservation}
\subsection{Rank Preservation}
\label{sec:theoretical_analysis}

In this section, we analyze the dynamics of the KV cache rank during inference. We define the rank of head $h$ as $r_h(t) := \mathrm{Rank}_{\text{eff}}(\mathbf{S}_h(t))$ and the multi-head rank vector as $\mathbf{r}(t) := [r_1(t), \dots, r_H(t)] \in \mathbb{R}^H$.
Empirically, we observe that $\mathbf{r}(t)$ exhibits high directional preservation. To formalize this, we introduce the following assumption based on the properties of rank-one updates.

\begin{assumption}[Non-degenerate Monotone Rank Growth]
\label{ass:monotone_rank_main}
For all attention heads $h$ and inference steps $t$, the rank is non-decreasing, i.e., $\delta_h(t)=r_h(t) - r_h(t-1) \ge 0$.
\end{assumption}

Unlike the algebraic rank which is strictly integer-valued and non-decreasing for update additions, the effective rank is continuous and sensitive to the spectral distribution, allowing for local fluctuations.
However, as shown in Figure \ref{fig:layer_0_32heads_rank}, we empirically observe a stratification of heads into two distinct categories:
\begin{itemize}
    \item \textbf{Accumulating Heads:} These heads exhibit a sustained rank increase ($\delta_h(t) > 0$) until reaching the upper bound derived in Theorem \ref{thm1}, at which point the increment vanishes ($\delta_h(t) = 0$).
    \item \textbf{Saturated Heads:} These heads fluctuate around a specific stable value (bounded by their intrinsic capacity). For these heads, the long-term average increment is zero ($\mathbb{E}[\delta_h(t)] \approx 0$).
\end{itemize}

\begin{theorem}[Recursive Stability of Rank Vector Cosine Similarity]
\label{thm:recursive_cosine_main}
Let $\varepsilon \in (0, 1)$. Under Assumption \ref{ass:monotone_rank_main}, if the cosine similarity between the rank vectors at step $t-1$ and $t$ is sufficiently high, i.e.,
\begin{equation}
\cos\big(\mathbf{r}(t-1), \mathbf{r}(t)\big) \ge 1 - \varepsilon,
\end{equation}
then the similarity is preserved in the subsequent step:
\begin{equation}
\cos\big(\mathbf{r}(t), \mathbf{r}(t+1)\big) \ge 1 - \varepsilon.
\end{equation}
\end{theorem}

A detailed proof is provided in Appendix \ref{sec:rank_preservation}
Theorem \ref{thm:recursive_cosine_main} provides a theoretical guarantee for our observation: once the rank distribution across heads stabilizes (high cosine similarity), it remains stable throughout the generation process. This suggests that the relative importance or complexity capacity of different heads is determined early and preserved.

\subsection{Norm Preservation}
\label{sec:norm_preservation}

In this section, we analyze the directional stability of the nuclear norm vector $\mathbf{n}(t) = (\|\mathbf{S}_1(t)\|_*, \dots, \|\mathbf{S}_H(t)\|_*) \in \mathbb{R}^H$.
Recall the update rule $\mathbf{S}_h(t) = \mathbf{S}_h(t-1) + \mathbf{k}_t^{(h)} (\mathbf{v}_t^{(h)})^\top$.
The evolution of the norm vector can be modeled as $\mathbf{n}(t) = \mathbf{n}(t-1) + \boldsymbol{\Delta}(t)$, where $\boldsymbol{\Delta}(t)$ represents the increment in nuclear norm at step $t$.

To derive the stability guarantee, we introduce a natural assumption regarding the boundedness of the input representations.

\begin{assumption}[Bounded Key/Value Magnitudes]
\label{ass:bounded_kv_main}
There exists a constant $B > 0$ such that the norm of the update vector in the nuclear norm space is bounded for all time steps $t$:
\begin{equation}
\|\boldsymbol{\Delta}(t)\|_2 
\;\le\; \sqrt{\sum_{h=1}^H \big(\|\mathbf{k}_t^{(h)}\|_2 \|\mathbf{v}_t^{(h)}\|_2\big)^2} 
\;\le\; B.
\end{equation}
\end{assumption}

This assumption follows directly from the fact that the increment in nuclear norm is Lipschitz continuous with respect to the rank-one update, bounded by the product of the key and value norms.
Based on this, we establish that the directional stability of $\mathbf{n}(t)$ is a consequence of the accumulation effect: as the total norm grows, the relative impact of new tokens diminishes.

\begin{theorem}[Directional Stability via Relative Step Size Decay]
\label{thm:norm_stability_main}
Let $r_t := \frac{B}{\|\mathbf{n}(t-1)\|_2}$ be the worst-case relative step size. Under Assumption \ref{ass:bounded_kv_main}, if $r_t < 1$, the cosine similarity between the nuclear norm vectors at adjacent steps is lower-bounded by:
\begin{equation}
\cos\big(\mathbf{n}(t-1), \mathbf{n}(t)\big) \;\ge\; \frac{1 - r_t}{1 + r_t}.
\label{eq:cos_bound_main}
\end{equation}
\end{theorem}

Detailed derivations and the proof of the nuclear norm increment bound are provided in Appendix \ref{thm_norm_ore}
\begin{remark}[Asymptotic Stability]
Theorem \ref{thm:norm_stability_main} implies that as the inference proceeds, the accumulated norm $\|\mathbf{n}(t-1)\|_2$ typically grows (representing the total information capacity stored in the cache), causing the ratio $r_t$ to approach zero. Consequently, the lower bound in \eqref{eq:cos_bound_main} approaches $1$, theoretically explaining the empirically observed high cosine similarity.
\end{remark}

\begin{figure*}[t] 
    \centering
    
    \begin{subfigure}{0.23\textwidth}
        \includegraphics[width=\linewidth]{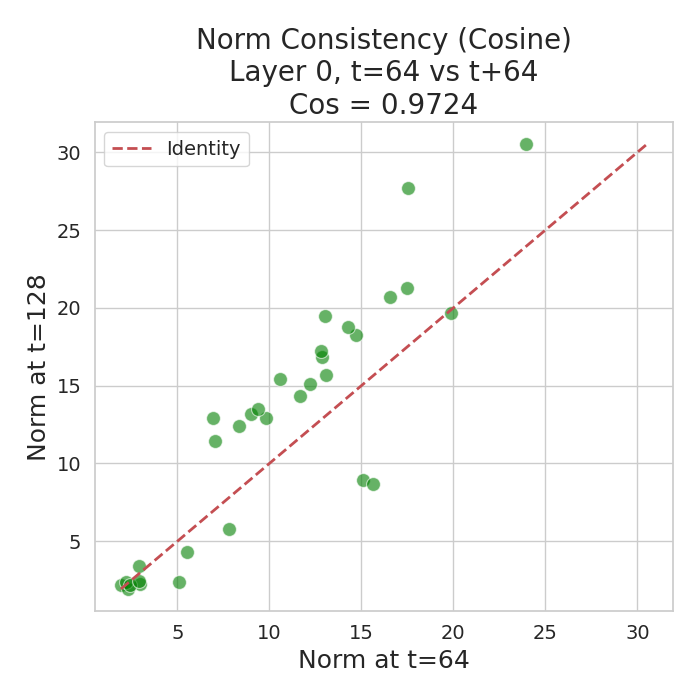}
        \caption{norm consistency between t64 and t128 in layer0}
    \end{subfigure}
    \hfill
    \begin{subfigure}{0.23\textwidth}
        \includegraphics[width=\linewidth]{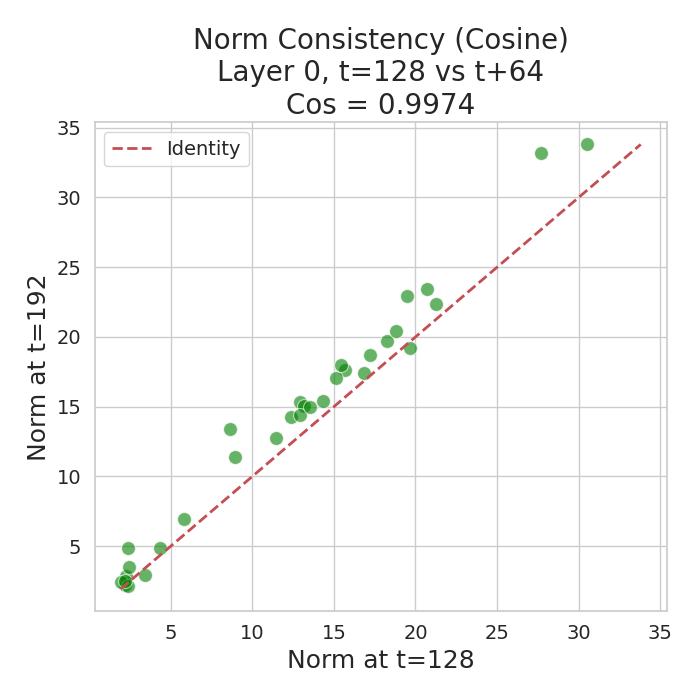}
        \caption{norm consistency between t128 and t192 in layer0}
    \end{subfigure}
    %
    \hspace{0.02\textwidth} 
    %
    \begin{subfigure}{0.23\textwidth}
        \includegraphics[width=\linewidth]{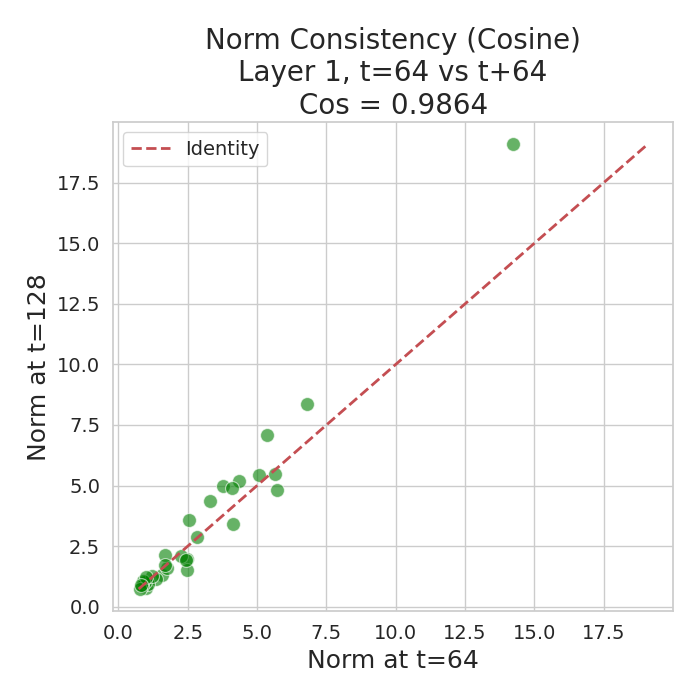}
        \caption{norm consistency between t64 and t128 in layer1}
    \end{subfigure}
    \hfill
    \begin{subfigure}{0.23\textwidth}
        \includegraphics[width=\linewidth]{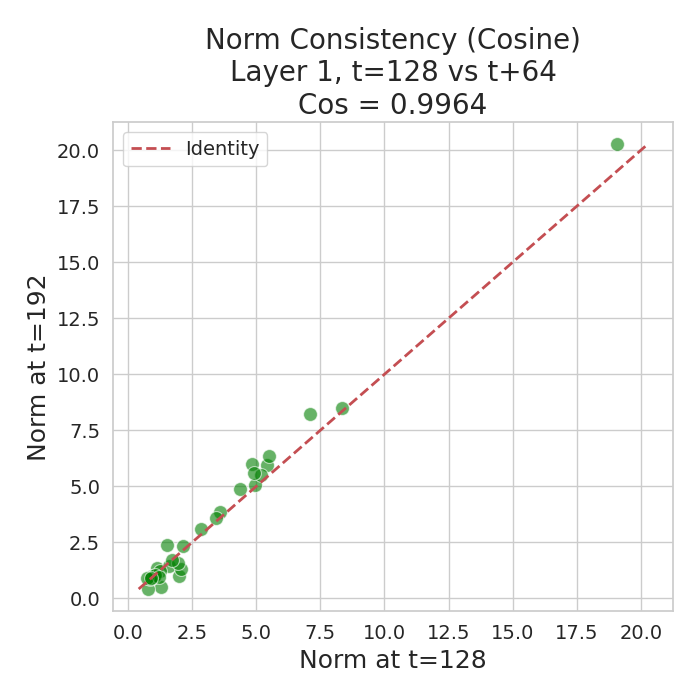}
        \caption{norm consistency between t128 and t192 in layer1}
    \end{subfigure}

    \par\bigskip 

    \begin{subfigure}{0.23\textwidth}
        \includegraphics[width=\linewidth]{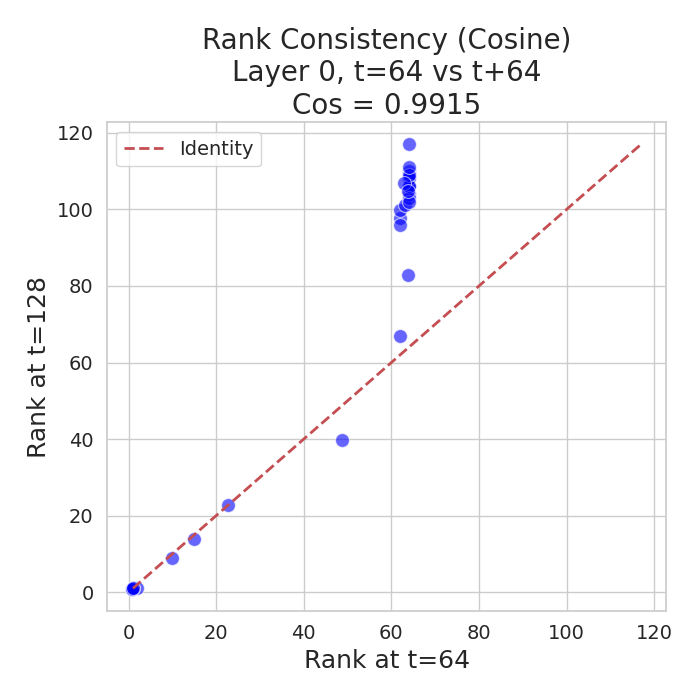}
        \caption{rank consistency between t64 and t128 in layer0}
    \end{subfigure}
    \hfill
    \begin{subfigure}{0.23\textwidth}
        \includegraphics[width=\linewidth]{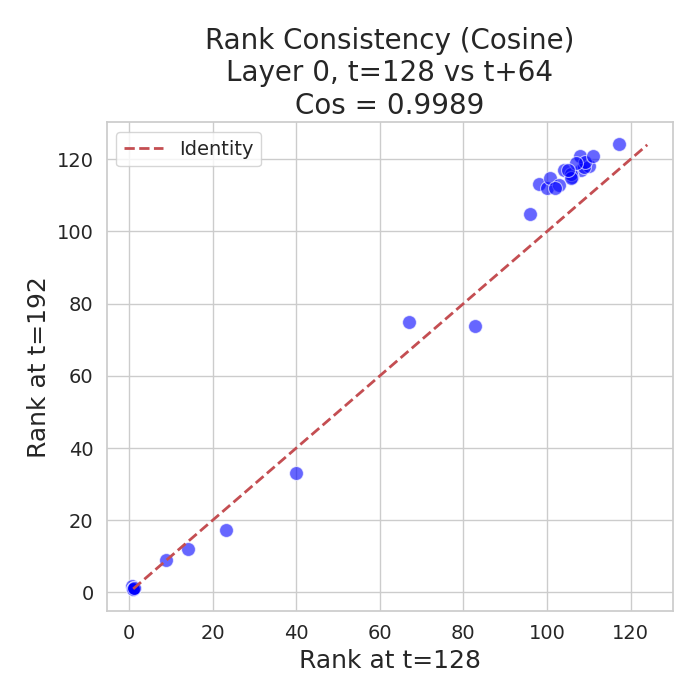}
        \caption{rank consistency between t128 and t192 in layer0}
    \end{subfigure}
    %
    \hspace{0.02\textwidth}
    %
    \begin{subfigure}{0.23\textwidth}
        \includegraphics[width=\linewidth]{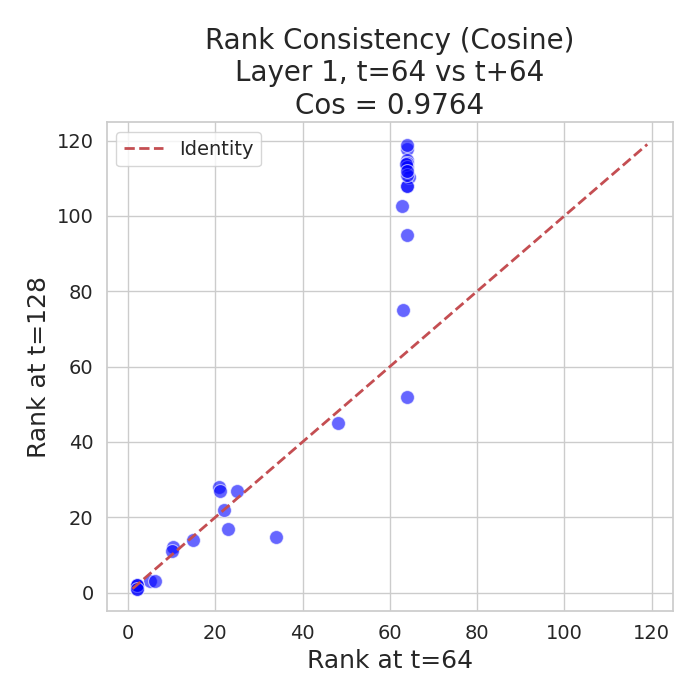}
        \caption{rank consistency between t64 and t128 in layer1}
    \end{subfigure}
    \hfill
    \begin{subfigure}{0.23\textwidth}
        \includegraphics[width=\linewidth]{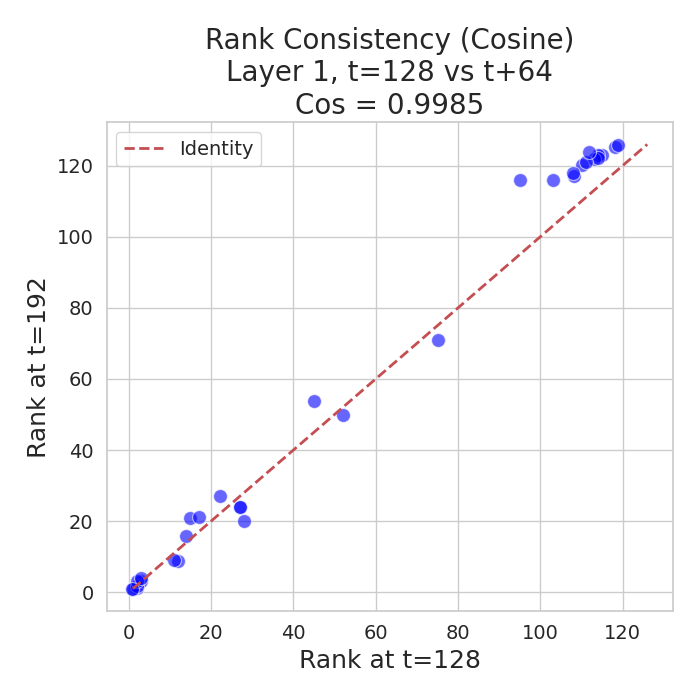}
        \caption{rank consistency between t128 and t192 in layer1}
    \end{subfigure}


    \caption{\textbf{Strong Temporal Consistency of State Norm and Rank.} 
    The scatter plots illustrate the stability of \textit{Nuclear Norm} (Top Row, \textbf{a-d}) and \textit{Effective Rank} (Bottom Row, \textbf{e-h}) across different time steps ($t$ vs. $t+\Delta t$) in representative layers. 
    Each point represents an attention head. The red dashed line denotes the identity mapping ($y=x$). 
    The extremely high cosine similarity scores ($\text{Cos\_Similarity} > 0.97$) and the tight alignment of points along the diagonal confirm the strong temporal stability and rank-order preservation of the state properties: heads with relatively higher rank or norm values maintain their magnitude relative to others throughout the inference process.}
    \label{fig:temporal_consistency}
\end{figure*}


\begin{table*}[t]
\centering
\setlength{\tabcolsep}{5pt}
\begin{tabular*}{\textwidth}{@{\extracolsep{\fill}}llccc}
\toprule
& & \multicolumn{3}{c}{\textbf{Task Categories (Accuracy \%)}} \\
\cmidrule(lr){3-5}
\textbf{Model} & \textbf{Method} & \textbf{Math-500} & \textbf{GSM8K} & \textbf{NIAH} \\
\midrule

\multirow{3}{*}{\textbf{Qwen3-Next-Instruct}} 
  & Origin                & 75.0 & 100.0 & 93.8 \\
  & Prune High-Rank       & 68.8 & 96.9 & 90.6 \\
  & Prune Low-Rank        & 25.0 & 75.0 & 46.9 \\
  
\midrule

\multirow{3}{*}{\textbf{Qwen3-Next-Thinking}} 
  & Origin                & 78.1 & 96.9 & 90.6 \\
  & Prune High-Rank       & 65.6 & 90.6 & 62.5 \\
  & Prune Low-Rank        & 25.0 & 75.0 & 43.8 \\

\bottomrule
\end{tabular*}
\caption{\textbf{Impact of Rank-based Pruning on Model Performance.} We evaluate the resilience of Qwen3-Next-Instruct and Qwen3-Next-Thinking variants across mathematical reasoning (Math-500, GSM8K) and long-context retrieval (NIAH) tasks. Across all benchmarks, the results demonstrate a stark performance asymmetry: while the models largely maintain accuracy when high-rank heads are removed, they suffer catastrophic performance degradation upon pruning low-rank heads, underscoring the indispensable role of low-rank regimes in linear attention states.}
\label{tab:pruning_ablation}
\end{table*}

\section{Joint Rank-Norm Pruning}
\label{sec:pruning}

How should we leverage the discovered Rank Stratification for efficient inference? Intuitively, one might assume that High-Rank heads, which utilize the full state capacity ($d=128$), are the primary carriers of information, while Low-Rank heads are candidates for pruning. 
However, our empirical investigation reveals a surprising paradox that challenges this conventional wisdom.

\subsection{Empirical Motivation: The High-Rank Redundancy}
To determine the functional importance of each group, we conducted a rigorous ablation study on the Qwen3-Next model using the \textit{Needle In A Haystack} (NIAH\citep{LLMTest_NeedleInAHaystack}) benchmark and mathematical reasoning tasks (MATH-500\citep{lightman2023letsverifystepstep,hendrycks2021measuringmathematicalproblemsolving} and GSM8K\citep{cobbe2021gsm8k}). We defined two experimental groups:
\begin{itemize}
    \item \textbf{High-Rank Group:} Heads with effective rank consistently near saturation ($r \approx d$).
    \item \textbf{Low-Rank Group:} Heads maintaining low effective rank ($r \ll d$) throughout inference.
\end{itemize}

As shown in Table~\ref{tab:pruning_ablation}, the results are striking:
\begin{itemize}
    \item \textbf{Pruning Low-Rank Heads:} Leads to catastrophic performance collapse. For instance, accuracy on the NIAH task drops precipitously from 93.8\% to 46.9\%. This identifies Low-Rank heads as \textit{indispensable} for retrieval and reasoning.
    \item \textbf{Pruning High-Rank Heads:} Results in negligible performance degradation (93.8\% $\to$ 90.6\%), despite removing heads that theoretically store more bits of information.
\end{itemize}

\textbf{Analysis.} We hypothesize that this phenomenon arises from State Saturation. High-Rank heads likely operate in an "over-saturated" regime where they attempt to store all incoming tokens indiscriminately, leading to noise accumulation and interference. In contrast, Low-Rank heads may implement selective attention mechanisms, maintaining a clean, low-dimensional subspace dedicated to specific retrieval patterns (e.g., induction heads).

\subsection{Methodology: Identifying Saturated Heads}
Based on this insight, we propose Joint Rank-Norm Pruning (JRNP), a strategy designed to identify and remove these redundant, over-saturated heads. We formulate a Saturation Score $\mathcal{S}_{h}$ to quantify the degree of redundancy for each head $h$:

\begin{equation}
    \mathcal{S}_{h} = \alpha \cdot \frac{\bar{r}_h}{d} + (1-\alpha) \cdot \frac{\bar{n}_h}{\max_j \bar{n}_j}
    \label{eq:score}
\end{equation}

where $\bar{r}_h$ and $\bar{n}_h$ are the average effective rank and nuclear norm. 
Under this metric, a \textbf{higher score indicates higher saturation/redundancy}. We target heads with the top-$k$\% scores for pruning, effectively removing the "noisy accumulators" while preserving the "clean specialists."

\subsection{Efficiency Gains and Robustness Verification}
By permanently pruning the heads identified as "saturated" (i.e., the High-Rank heads), we achieve a direct reduction in the KV-state memory footprint and recurrent update computations. 
Table~\ref{tab:pruning_ablation} empirically validates the safety of this strategy. 
Specifically, for Qwen3-Next-Instruct, removing High-Rank heads results in minimal performance degradation, maintaining 96.9\% accuracy on GSM8K and 90.6\% on NIAH (compared to 100.0\% and 93.8\% baseline). 
In stark contrast, pruning Low-Rank heads leads to a catastrophic collapse, with NIAH accuracy dropping to 46.9\%. 
This performance asymmetry strongly confirms our hypothesis: High-Rank heads represent redundant capacity that can be traded for efficiency, whereas Low-Rank heads serve as the indispensable backbone for reasoning and retrieval.

\subsection{Analysis of Long-Context Collapse}
Pruning Low-Rank heads triggers \textit{Long-Context Collapse}, where retrieval capability degrades to near-random levels (as evidenced by NIAH). This confirms that Low-Rank heads function as essential spectral anchors within low-dimensional subspaces to lock onto critical information, a property lacking in saturated High-Rank heads. See Appendix~\ref{app:case_study} for a concrete case study on the resulting reasoning degradation.


\section{Related Work}
\label{sec:related_work}

\textbf{Efficient Transformers and State Space Duality.}
While Transformers \citep{vaswani2017attention} dominate foundation models, their quadratic time and memory complexity $O(L^2)$ necessitate efficient alternatives for long-context modeling. Linear Attention (LA) mechanisms \citep{katharopoulos2020transformers} reduce this complexity to $O(L)$ by utilizing the kernel trick to rewrite attention as a Recurrent Neural Network (RNN). Concurrently, State Space Models (SSMs) like Mamba \citep{gu2024mamba} have evolved from structured matrices to selective gating mechanisms. Recently, \citet{dao2024transformers} established the theoretical framework of Structured State Space Duality (SSD), proving the mathematical equivalence between linear attention and SSMs under specific conditions. Our work builds upon modern Gated Linear Attention architectures, specifically Gated DeltaNet \citep{yang2024gated} and Qwen3-Next, which employ data-dependent gates to dynamically control memory updates. Unlike prior works that focus on architecture design, we analyze the internal dynamics of the recurrent state matrix $S_t$, treating it as the dual representation of the model's memory.

\textbf{The Low-Rank Dilemma and Memory Capacity.}
A fundamental limitation of condensing history into a fixed-size state $S_t \in \mathbb{R}^{D \times D}$ is the capacity bottleneck. \citet{fan2025breaking} formally identified the Low-Rank Dilemma, proving that standard linear attention updates tend to produce a low-rank state matrix, thereby restricting the model's expressivity compared to Full Rank Softmax Attention. Similarly, \citet{arora2024simple} highlighted the recall-throughput tradeoff, showing that low-dimensional states struggle with associative recall tasks, a finding corroborated by \citet{jelassi2024repeat} who demonstrated that Transformers outperform SSMs in copying tasks. Our work provides an empirical extension to these theoretical foundations. While \citet{fan2025breaking} suggest increasing head count to mitigate this bottleneck, we use spectral analysis to reveal that models spontaneously adopt a Rank Specialization strategy.


\textbf{Spectral Analysis in Deep Learning.}
Singular Value Decomposition (SVD) has long served as a powerful lens for interpreting the internal representations of deep neural networks. Early studies utilized spectral density analysis to investigate the implicit regularization and generalization capabilities of weight matrices \citep{martin2021implicit}. In the realm of model compression, low-rank approximations of weight matrices have been widely adopted to reduce redundancy without compromising performance \citep{denton2014exploiting}. More relevant to our context are investigations into the representation collapse or anisotropy problem in Transformers \citep{ethayarajh2019contextual, gao2019representation}, where token embeddings tend to degenerate into a narrow cone, limiting their expressivity. However, these works predominantly focus on \textit{static} parameter matrices or \textit{transient} token embeddings. Our work extends this spectral perspective to the \textit{dynamic recurrent state} $\mathbf{S}(t)$ of SSMs. We utilize Effective Rank not merely as a static metric of compressibility, but as a real-time indicator of memory accumulation dynamics, revealing how the model dynamically allocates its spectral budget over time.


\textbf{KV Cache Pruning in LLMs.} 
To mitigate memory bottlenecks, strategies like H2O \citep{zhang2023h2o} and StreamingLLM \citep{xiao2023efficient} leverage "heavy hitter" and "attention sink" phenomena for token eviction. Subsequent methods, including SnapKV \citep{li2024snapkv}, PyramidKV \citep{zhong2025zigzagkv}, and ShadowKV \citep{sun2024shadowkv}, further refine token-level sparsity or offloading. 
However, these heuristics fail in Linear Attention's recurrent paradigm, where context is compressed into fixed-size states. 
Our work bridges this gap by shifting from token-discarding to a \textit{state-centric} paradigm: we permanently prune redundant \textit{attention heads} based on \textit{State Rank Stratification}, utilizing the spectral properties of the recurrence.




\section{Conclusion}
\label{sec:conclusion}

In this work, we presented a systematic spectral analysis of recurrent state dynamics in Linear Attention LLMs, uncovering the phenomenon of \textbf{State Rank Stratification}. Our theoretical framework elucidates this emergence through two mechanisms: Rank Saturation, where heads reach subspace capacities, and Norm Accumulation, which ensures temporal invariance via cache inertia. Crucially, we identified that preserving specific low-rank dynamics is essential to preventing Long-Context Collapse. Leveraging these insights, we introduced Joint Rank-Norm Pruning (JRNP), a training-free strategy that achieves a 38.9\% reduction in KV-state memory with negligible performance loss. Our findings challenge the pursuit of high-rank state representations and suggest that future architectures should explicitly incorporate low-rank inductive biases to optimize the efficiency-retrieval trade-off.

\section{Discussion}
\label{sec:discussion}

While our work establishes the existence of State Rank Stratification and its utility in pruning, several directions remain for future exploration.

\textbf{Theoretical Mechanism of Invariance.} 
While we empirically verify that head roles are independent of data domains, a rigorous derivation remains to be established. We hypothesize this invariance originates from the singular value spectra of pre-trained weights ($W_K, W_V$), which act as spectral filters constraining state dynamics. Future work could leverage Random Matrix Theory to mathematically link static weight distributions to runtime rank bounds.

\textbf{Architectural Generalization.} 
Our analysis centers on Qwen3-Next. However, diverse sub-quadratic architectures (e.g., Mamba, RWKV, RetNet) possess distinct gating and normalization mechanisms that may modulate rank stratification. Expanding our spectral analysis to these families is crucial to validate the universality of our findings and adapt the Joint Rank-Norm Pruning strategy accordingly.

\clearpage

\section*{Impact Statement}
This paper presents work whose goal is to advance the field of Machine
Learning. There are many potential societal consequences of our work, none
which we feel must be specifically highlighted here.

\nocite{langley00}

\bibliography{example_paper}
\bibliographystyle{icml2026}

\newpage
\appendix
\onecolumn



\section{Rank Upper Bound}
\label{sec:rank_upper_bound}

In this section, we formally derive the upper bound of the rank of the state matrix $\mathbf{S}(t) \in \mathbb{R}^{d \times d}$. By the fundamental theorem of linear algebra, the condition $\mathrm{rank}(\mathbf{S}(t)) \le d$ holds trivially. It remains to prove that $\mathrm{rank}(\mathbf{S}(t)) \le t$.

\begin{theorem}[Rank Upper Bound for Linear Attention]
For both Standard Linear Attention and DeltaNet, the rank of the state matrix satisfies $\mathrm{rank}(\mathbf{S}(t)) \le \min(t, d)$ for all $t \ge 0$.
\end{theorem}

\begin{proof}
We proceed by mathematical induction.
\textbf{Base Case:} At $t=0$, $\mathbf{S}_0 = \mathbf{0}$, thus $\mathrm{rank}(\mathbf{S}_0) = 0 \le 0$.

\paragraph{Case 1: Standard Linear Attention.}
The update rule is given by $\mathbf{S}(t) = \alpha_t \mathbf{S}(t-1) + \beta_t \mathbf{v}_t \mathbf{k}_t^\top$.
\textbf{Inductive Step:} Assume $\mathrm{rank}(\mathbf{S}(t-1)) \le t-1$. Since scalar multiplication does not increase rank (for non-zero scalars) and the term $\mathbf{v}_t \mathbf{k}_t^\top$ is an outer product of two vectors (maximum rank of 1), using the subadditivity of matrix rank ($\mathrm{rank}(\mathbf{A}+\mathbf{B}) \le \mathrm{rank}(\mathbf{A}) + \mathrm{rank}(\mathbf{B})$), we derive:
\begin{equation}
    \begin{aligned}
        \mathrm{rank}(\mathbf{S}(t)) &= \mathrm{rank}(\alpha_t \mathbf{S}(t-1) + \beta_t \mathbf{v}_t \mathbf{k}_t^\top) \\
        &\le \mathrm{rank}(\mathbf{S}(t-1)) + \mathrm{rank}(\mathbf{v}_t \mathbf{k}_t^\top) \\
        &\le (t-1) + 1 \\
        &= t.
    \end{aligned}
\end{equation}

\paragraph{Case 2: DeltaNet.}
The update rule incorporating the delta rule is $\mathbf{S}(t) = (\mathbf{I} - \beta_t \mathbf{k}_t \mathbf{k}_t^\top)\mathbf{S}(t-1) + \beta_t \mathbf{v}_t \mathbf{k}_t^\top$. By rearranging terms, we can express this as an additive update:
\begin{equation}
    \mathbf{S}(t) = \mathbf{S}(t-1) + \beta_t (\mathbf{v}_t - \mathbf{S}(t-1) \mathbf{k}_t) \mathbf{k}_t^\top.
\end{equation}
Let the error vector be $\mathbf{e}_t = \beta_t (\mathbf{v}_t - \mathbf{S}(t-1) \mathbf{k}_t) \in \mathbb{R}^{d}$. The update simplifies to $\mathbf{S}(t) = \mathbf{S}(t-1) + \mathbf{e}_t \mathbf{k}_t^\top$.
Similar to Case 1, the term $\mathbf{e}_t \mathbf{k}_t^\top$ represents a rank-1 update. Applying the same inductive logic:
\begin{equation}
    \begin{aligned}
        \mathrm{rank}(\mathbf{S}(t)) &\le \mathrm{rank}(\mathbf{S}(t-1)) + \mathrm{rank}(\mathbf{e}_t \mathbf{k}_t^\top) \\
        &\le (t-1) + 1 \\
        &= t.
    \end{aligned}
\end{equation}
\end{proof}

\section{Rank Preservation}
\label{sec:rank_preservation}

We investigate the KV cache update of Linear Attention Transformers during the inference phase. For each attention head $h \in \{1, \dots, H\}$, the cache matrix $\mathbf{S}_h(t) \in \mathbb{R}^{d \times d}$ is updated recursively via a rank-one outer product:
\begin{equation}
\mathbf{S}_h(t) \;=\; \mathbf{S}_h(t-1) + \mathbf{k}^{(h)}_t \big(\mathbf{v}^{(h)}_t\big)^\top,\qquad \mathbf{S}_h(0)=\mathbf{0},
\label{eq:kv_update}
\end{equation}
where $\mathbf{k}^{(h)}_t, \mathbf{v}^{(h)}_t \in \mathbb{R}^{d}$ denote the key and value vectors for the $t$-th token in head $h$, respectively.
We define the rank of each head and the multi-head rank vector as:
\begin{equation}
r_h(t) \;:=\; \mathrm{rank}\big(\mathbf{S}_h(t)\big),\qquad 
\mathbf{r}(t) \;:=\; \big(r_1(t), \dots, r_H(t)\big) \in \mathbb{R}^{H}.
\label{eq:rank_vector_def}
\end{equation}
Empirically, we observe that the cosine similarity between rank vectors $\mathbf{r}(t)$ at different inference steps is typically very close to $1$. This phenomenon holds for both $t < d$ and $t \ge d$. Furthermore, we observe a stratification where some heads reach a global rank upper bound while others remain at a persistently low rank.

\paragraph{Intrinsic Rank Upper Bound per Head.}
For each head $h$, we define the spanning spaces of the keys and values:
\begin{equation}
\mathcal{K}_h(t) \coloneqq \mathrm{span}\{\mathbf{k}^{(h)}_1, \dots, \mathbf{k}^{(h)}_t\},\qquad
\mathcal{V}_h(t) \coloneqq \mathrm{span}\{\mathbf{v}^{(h)}_1, \dots, \mathbf{v}^{(h)}_t\},
\label{eq:kv_spans}
\end{equation}
and define the intrinsic upper bound for head $h$:
\begin{equation}
m_h \;:=\; \min\Big(\sup_{t}\dim \mathcal{K}_h(t),\ \sup_{t}\dim \mathcal{V}_h(t)\Big),
\qquad \text{thus } m_h \le d.
\label{eq:mh_def}
\end{equation}
When $m_h = d$, the head has the capacity to reach the global upper bound; when $m_h \ll d$, the rank of this head will saturate at a lower level.

\begin{lemma}[Structural Rank Bound via Key/Value Subspaces]
\label{lem:span_rank_bound}
For any head $h$ and any $t \ge 1$,
\begin{equation}
\mathrm{rank}\big(\mathbf{S}_h(t)\big) \le \min\big(\dim \mathcal{K}_h(t),\ \dim \mathcal{V}_h(t)\big) \le m_h \le d.
\label{eq:rank_le_mh}
\end{equation}
\end{lemma}

\begin{proof}
From \eqref{eq:kv_update}, $\mathbf{S}_h(t) = \sum_{i=1}^t \mathbf{k}^{(h)}_i (\mathbf{v}^{(h)}_i)^\top$.
For any $\mathbf{x} \in \mathbb{R}^d$,
\begin{equation}
\mathbf{S}_h(t)\mathbf{x} = \sum_{i=1}^t \mathbf{k}^{(h)}_i \big((\mathbf{v}^{(h)}_i)^\top \mathbf{x}\big).
\end{equation}
The RHS is a linear combination of $\{\mathbf{k}^{(h)}_1, \dots, \mathbf{k}^{(h)}_t\}$, implying $\mathrm{Col}(\mathbf{S}_h(t)) \subseteq \mathcal{K}_h(t)$. Thus, $\mathrm{rank}(\mathbf{S}_h(t)) \le \dim\mathcal{K}_h(t)$.
Similarly, considering $\mathbf{S}_h(t)^\top = \sum_{i=1}^t \mathbf{v}^{(h)}_i (\mathbf{k}^{(h)}_i)^\top$, we have $\mathrm{Row}(\mathbf{S}_h(t)) \subseteq \mathcal{V}_h(t)$, implying $\mathrm{rank}(\mathbf{S}_h(t)) \le \dim\mathcal{V}_h(t)$.
Combining these yields \eqref{eq:rank_le_mh}.
\end{proof}

\paragraph{Rank Increment and Recursion.}
We define the rank increment for each head and the increment vector as:
\begin{equation}
\delta_h(t) \;:=\; r_h(t) - r_h(t-1),\qquad
\boldsymbol\delta(t) \coloneqq \big(\delta_1(t), \dots, \delta_H(t)\big) \in \mathbb{R}^H.
\label{eq:delta_def}
\end{equation}
The rank vector thus satisfies the recursion:
\begin{equation}
\mathbf{r}(t) = \mathbf{r}(t-1) + \boldsymbol\delta(t).
\label{eq:r_recursion}
\end{equation}

\begin{lemma}[Step-wise Rank Growth Bound]
\label{lem:rank_step_bound}
For any head $h$ and any $t \ge 1$,
\begin{equation}
r_h(t) \le r_h(t-1) + 1.
\label{eq:rank_plus_one}
\end{equation}
\end{lemma}

\begin{proof}
By the subadditivity of rank:
\begin{equation}
\mathrm{rank}\big(\mathbf{S}_h(t)\big)
= \mathrm{rank}\big(\mathbf{S}_h(t-1) + \mathbf{k}^{(h)}_t(\mathbf{v}^{(h)}_t)^\top\big)
\le \mathrm{rank}\big(\mathbf{S}_h(t-1)\big) + \mathrm{rank}\big(\mathbf{k}^{(h)}_t(\mathbf{v}^{(h)}_t)^\top\big).
\end{equation}
Assuming $\mathbf{k}^{(h)}_t \neq \mathbf{0}$ and $\mathbf{v}^{(h)}_t \neq \mathbf{0}$, the outer product rank is 1, yielding \eqref{eq:rank_plus_one}.
\end{proof}

We adopt the following assumption, consistent with empirical observations:

\begin{assumption}[Non-degenerate Monotone Rank Growth]
\label{ass:monotone_rank}
For all heads $h$ and steps $t$, $r_h(t) \ge r_h(t-1)$.
\end{assumption}

Assumption \ref{ass:monotone_rank} and Lemma \ref{lem:rank_step_bound} imply $\delta_h(t) \in \{0, 1\}$.
Furthermore, by Lemma \ref{lem:span_rank_bound}, $r_h(t) \le m_h$. Thus, once a head reaches its intrinsic upper bound, it saturates permanently: if $r_h(t_0) = m_h$, then $r_h(t) = m_h$ for all $t \ge t_0$.

\begin{lemma}[Monotonic Non-Increasing Support and Norm of Increment Vector]
\label{lem:delta_monotone}
Under Assumption \ref{ass:monotone_rank}, for any $t \ge 1$,
\begin{equation}
\boldsymbol\delta(t+1) \le \boldsymbol\delta(t) \quad \text{(element-wise)},
\qquad \text{consequently} \qquad
\|\boldsymbol\delta(t+1)\|_2 \le \|\boldsymbol\delta(t)\|_2.
\label{eq:delta_monotone}
\end{equation}
\end{lemma}

\begin{proof}
We prove this element-wise. Fix an arbitrary head $h$.
From Lemma \ref{lem:span_rank_bound}, $r_h(t) \le m_h$ for all $t$.
Given Assumption \ref{ass:monotone_rank}, the sequence $\{r_h(t)\}_{t \ge 0}$ is monotonically non-decreasing and bounded by $m_h$.
Therefore, there exists a saturation time step $t_h^\star$ such that for all $t \ge t_h^\star$, $r_h(t) = m_h$, which implies $\delta_h(t) = 0$.

Next, we show that once $\delta_h(t)$ becomes $0$, it remains $0$ forever.
If $\delta_h(t) = 0$ for some $t$, then $r_h(t) = r_h(t-1)$.
If $\delta_h(t+1) = 1$ were to occur, then $r_h(t+1) = r_h(t) + 1$.
Repeating this logic allows $r_h(\cdot)$ to grow indefinitely.
However, since $r_h(\cdot)$ is bounded by $m_h$, growth can only occur a finite number of times.
Once $m_h$ is reached, the increment must be $0$.
Therefore, there is a critical moment after which the component is identically $0$.
Equivalently, for all $t$, $\delta_h(t+1) \le \delta_h(t)$ (since $\delta_h(t) \in \{0, 1\}$, this inequality asserts that a transition from $0$ back to $1$ is impossible).
Since $h$ is arbitrary, the element-wise inequality $\boldsymbol\delta(t+1) \le \boldsymbol\delta(t)$ holds.

Finally, since the vector is non-increasing element-wise with non-negative components, the $\ell_2$ norm is also non-increasing, i.e., $\|\boldsymbol\delta(t+1)\|_2 \le \|\boldsymbol\delta(t)\|_2$.
\end{proof}

\paragraph{Recursive Stability of Cosine Similarity.}
We now prove that if the cosine similarity between adjacent rank vectors is close to $1$ at a certain step, it remains close to $1$ in the subsequent step.

\begin{lemma}[Cosine Lower Bound for Non-negative Vectors]
\label{lem:cos_lower_bound_simple}
Let $\mathbf{a}, \mathbf{u} \in \mathbb{R}^H$ be element-wise non-negative: $\mathbf{a} \ge \mathbf{0}, \mathbf{u} \ge \mathbf{0}$. Then:
\begin{equation}
\cos(\mathbf{a}, \mathbf{a} + \mathbf{u})
\;\ge\;
\frac{\|\mathbf{a}\|_2}{\|\mathbf{a}\|_2 + \|\mathbf{u}\|_2}.
\label{eq:cos_lower_simple}
\end{equation}
\end{lemma}

\begin{proof}
By definition, $\cos(\mathbf{a}, \mathbf{a} + \mathbf{u}) = \frac{\mathbf{a}^\top(\mathbf{a} + \mathbf{u})}{\|\mathbf{a}\|_2 \|\mathbf{a} + \mathbf{u}\|_2}$.
Expanding the numerator: $\mathbf{a}^\top(\mathbf{a} + \mathbf{u}) = \|\mathbf{a}\|_2^2 + \mathbf{a}^\top\mathbf{u}$.
Since $\mathbf{a}, \mathbf{u} \ge \mathbf{0}$, we have $\mathbf{a}^\top\mathbf{u} \ge 0$, and thus $\mathbf{a}^\top(\mathbf{a} + \mathbf{u}) \ge \|\mathbf{a}\|_2^2$.
Using the triangle inequality $\|\mathbf{a} + \mathbf{u}\|_2 \le \|\mathbf{a}\|_2 + \|\mathbf{u}\|_2$ for the denominator, we obtain \eqref{eq:cos_lower_simple}.
\end{proof}

\begin{theorem}[Recursive Stability of Rank Vector Cosine Similarity]
\label{thm:recursive_cosine}
Under Assumption \ref{ass:monotone_rank}, let $0 < \varepsilon < 1$.
If for some $t \ge 1$:
\begin{equation}
\cos\big(\mathbf{r}(t-1), \mathbf{r}(t)\big) \ge 1 - \varepsilon,
\label{eq:assume_prev_cos}
\end{equation}
then:
\begin{equation}
\cos\big(\mathbf{r}(t), \mathbf{r}(t+1)\big) \ge 1 - \varepsilon.
\label{eq:conclude_next_cos}
\end{equation}
\end{theorem}

\begin{proof}
By \eqref{eq:r_recursion}, $\mathbf{r}(t) = \mathbf{r}(t-1) + \boldsymbol\delta(t)$ and $\mathbf{r}(t+1) = \mathbf{r}(t) + \boldsymbol\delta(t+1)$.
Under Assumption \ref{ass:monotone_rank}, $\mathbf{r}(\cdot) \ge \mathbf{0}$ and $\boldsymbol\delta(\cdot) \ge \mathbf{0}$ hold element-wise.

\textbf{Step 1: Derive upper bound for increment norm from \eqref{eq:assume_prev_cos}.}
From Lemma \ref{lem:cos_lower_bound_simple} (setting $\mathbf{a} = \mathbf{r}(t-1)$, $\mathbf{u} = \boldsymbol\delta(t)$):
\begin{equation}
\cos\big(\mathbf{r}(t-1), \mathbf{r}(t)\big)
\ge
\frac{\|\mathbf{r}(t-1)\|_2}{\|\mathbf{r}(t-1)\|_2 + \|\boldsymbol\delta(t)\|_2}.
\end{equation}
Combining with \eqref{eq:assume_prev_cos} and letting $A \coloneqq \|\mathbf{r}(t-1)\|_2$, $D \coloneqq \|\boldsymbol\delta(t)\|_2$, we have:
\begin{equation}
\frac{A}{A+D} \ge 1 - \varepsilon
\;\Rightarrow\;
A \ge (1 - \varepsilon)(A + D)
\;\Rightarrow\;
\|\boldsymbol\delta(t)\|_2 \le \frac{\varepsilon}{1 - \varepsilon} \|\mathbf{r}(t-1)\|_2.
\label{eq:delta_bound_from_cos}
\end{equation}

\textbf{Step 2: Lower bound the next cosine similarity.}
Applying Lemma \ref{lem:cos_lower_bound_simple} again (setting $\mathbf{a} = \mathbf{r}(t)$, $\mathbf{u} = \boldsymbol\delta(t+1)$):
\begin{equation}
\cos\big(\mathbf{r}(t), \mathbf{r}(t+1)\big)
\ge
\frac{\|\mathbf{r}(t)\|_2}{\|\mathbf{r}(t)\|_2 + \|\boldsymbol\delta(t+1)\|_2}.
\label{eq:next_cos_lower}
\end{equation}

\textbf{Step 3: Utilize monotonicity of increment vector norm.}
By Lemma \ref{lem:delta_monotone}, $\|\boldsymbol\delta(t+1)\|_2 \le \|\boldsymbol\delta(t)\|_2$.
Also, since $\mathbf{r}(t) = \mathbf{r}(t-1) + \boldsymbol\delta(t)$ with non-negative components, $\|\mathbf{r}(t)\|_2 \ge \|\mathbf{r}(t-1)\|_2$.
Thus, from \eqref{eq:delta_bound_from_cos}:
\begin{equation}
\|\boldsymbol\delta(t+1)\|_2
\le \|\boldsymbol\delta(t)\|_2
\le \frac{\varepsilon}{1 - \varepsilon} \|\mathbf{r}(t-1)\|_2
\le \frac{\varepsilon}{1 - \varepsilon} \|\mathbf{r}(t)\|_2.
\label{eq:delta_t1_bound}
\end{equation}

\textbf{Step 4: Conclusion.}
Substituting \eqref{eq:delta_t1_bound} into \eqref{eq:next_cos_lower}:
\begin{equation}
\cos\big(\mathbf{r}(t), \mathbf{r}(t+1)\big)
\ge
\frac{\|\mathbf{r}(t)\|_2}{\|\mathbf{r}(t)\|_2 + \frac{\varepsilon}{1 - \varepsilon}\|\mathbf{r}(t)\|_2}
=
\frac{1}{1 + \frac{\varepsilon}{1 - \varepsilon}}
=
1 - \varepsilon.
\end{equation}
This yields \eqref{eq:conclude_next_cos}.
\end{proof}

\begin{remark}[Relationship with $\mathrm{rank}(\mathbf{S}_h(t)) \le \min\{d, t\}$]
We know the global bound $\mathrm{rank}(\mathbf{S}_h(t)) \le \min\{d, t\}$. Lemma \ref{lem:span_rank_bound} provides a finer-grained, head-specific structural bound $m_h \le d$, explaining the stratification phenomenon where some heads saturate while others remain low-rank. Moreover, as $t$ increases, more heads reach their $m_h$ and stop growing, causing the norm of the increment vector $\boldsymbol\delta(t)$ to decrease progressively (Lemma \ref{lem:delta_monotone}). This reinforces the collinearity of adjacent rank vectors. Theorem \ref{thm:recursive_cosine} formally states that once the cosine similarity is close to $1$, it recursively remains close to $1$.
\end{remark}

\section{Norm Preservation}
\label{thm_norm_ore}
\subsection{Problem Setup and Empirical Observations}

Consider a layer of a Linear Attention Transformer containing $H$ attention heads.
For the $h$-th head ($h \in \{1, \dots, H\}$) at inference time $t$, the KV cache matrix is defined as:
\begin{equation}
\mathbf{S}_h(t) \;=\; \sum_{i=1}^{t} \mathbf{k}_i^{(h)} \bigl(\mathbf{v}_i^{(h)}\bigr)^\top
\qquad\Longleftrightarrow\qquad
\mathbf{S}_h(t) \;=\; \mathbf{S}_h(t-1) + \mathbf{k}_t^{(h)} \bigl(\mathbf{v}_t^{(h)}\bigr)^\top ,
\label{eq:kv_cache_update}
\end{equation}
where $\mathbf{k}_t^{(h)} \in \mathbb{R}^{d_k}$ and $\mathbf{v}_t^{(h)} \in \mathbb{R}^{d_v}$ are key and value vectors.

Let $\|\cdot\|_*$ denote the nuclear norm (sum of singular values) and $\|\cdot\|_2$ denote the Euclidean norm for vectors (or spectral norm for matrices).
We define the nuclear norm of each head and the concatenated vector for all heads as:
\begin{equation}
n_h(t) \;:=\; \|\mathbf{S}_h(t)\|_*,
\qquad
\mathbf{n}(t) \;:=\; \bigl(n_1(t), \dots, n_H(t)\bigr) \in \mathbb{R}^H.
\label{eq:nuclear_vector_def}
\end{equation}

Empirically, we observe that the cosine similarity between $\mathbf{n}(t)$ at different time steps is close to $1$, i.e.,
\begin{equation}
\cos\bigl(\mathbf{n}(t), \mathbf{n}(t')\bigr)
\;=\;
\frac{\langle \mathbf{n}(t), \mathbf{n}(t')\rangle}{\|\mathbf{n}(t)\|_2 \, \|\mathbf{n}(t')\|_2}
\;\approx\; 1.
\label{eq:empirical_observation}
\end{equation}
Furthermore, the cosine similarity between adjacent steps $t-1$ and $t$ is also very close to $1$.

\subsection{Fundamental Lemmas}

\begin{lemma}[Nuclear Norm of Rank-One Outer Product]
For any vectors $\mathbf{a} \in \mathbb{R}^{d_k}$ and $\mathbf{b} \in \mathbb{R}^{d_v}$,
\begin{equation}
\|\mathbf{a}\mathbf{b}^\top\|_* \;=\; \|\mathbf{a}\|_2 \, \|\mathbf{b}\|_2.
\label{eq:rank1_nuclear}
\end{equation}
\end{lemma}

\begin{proof}
The matrix $\mathbf{a}\mathbf{b}^\top$ has rank 1, so it has only one non-zero singular value $\sigma$.
By the definition of spectral norm:
\[
\|\mathbf{a}\mathbf{b}^\top\|_2
=
\max_{\|\mathbf{x}\|_2=1}\|(\mathbf{a}\mathbf{b}^\top)\mathbf{x}\|_2
=
\max_{\|\mathbf{x}\|_2=1}\|\mathbf{a}(\mathbf{b}^\top \mathbf{x})\|_2
=
\|\mathbf{a}\|_2 \max_{\|\mathbf{x}\|_2=1}|\mathbf{b}^\top \mathbf{x}|
=
\|\mathbf{a}\|_2 \|\mathbf{b}\|_2.
\]
Since the nuclear norm of a rank-1 matrix equals its unique singular value (which equals its spectral norm),
$\|\mathbf{a}\mathbf{b}^\top\|_* = \sigma = \|\mathbf{a}\mathbf{b}^\top\|_2 = \|\mathbf{a}\|_2 \|\mathbf{b}\|_2$.
\end{proof}

\begin{lemma}[Lipschitz Continuity of Nuclear Norm]
For any matrices $\mathbf{A}, \mathbf{B}$ of the same dimensions,
\begin{equation}
\bigl|\|\mathbf{A}+\mathbf{B}\|_* - \|\mathbf{A}\|_*\bigr| \;\le\; \|\mathbf{B}\|_*.
\label{eq:nuclear_lipschitz}
\end{equation}
\end{lemma}

\begin{proof}
By the triangle inequality for the nuclear norm,
\[
\|\mathbf{A}+\mathbf{B}\|_* \le \|\mathbf{A}\|_* + \|\mathbf{B}\|_*
\quad\Rightarrow\quad
\|\mathbf{A}+\mathbf{B}\|_* - \|\mathbf{A}\|_* \le \|\mathbf{B}\|_*.
\]
Conversely,
\[
\|\mathbf{A}\|_* = \|(\mathbf{A}+\mathbf{B}) - \mathbf{B}\|_* \le \|\mathbf{A}+\mathbf{B}\|_* + \|\mathbf{B}\|_*
\quad\Rightarrow\quad
\|\mathbf{A}\|_* - \|\mathbf{A}+\mathbf{B}\|_* \le \|\mathbf{B}\|_*.
\]
Combining these yields \eqref{eq:nuclear_lipschitz}.
\end{proof}

\subsection{Bounded Single-Step Increment in Head Space}

Define the single-step increment for each head:
\begin{equation}
\Delta_h(t) \;:=\; n_h(t) - n_h(t-1)
=
\|\mathbf{S}_h(t)\|_* - \|\mathbf{S}_h(t-1)\|_*,
\qquad
\boldsymbol{\Delta}(t) \coloneqq (\Delta_1(t), \dots, \Delta_H(t)).
\label{eq:delta_def_n}
\end{equation}
By \eqref{eq:kv_cache_update} and Lemma \ref{eq:nuclear_lipschitz},
\begin{equation}
|\Delta_h(t)|
=
\bigl|\|\mathbf{S}_h(t-1) + \mathbf{k}_t^{(h)}(\mathbf{v}_t^{(h)})^\top\|_* - \|\mathbf{S}_h(t-1)\|_*\bigr|
\le
\|\mathbf{k}_t^{(h)}(\mathbf{v}_t^{(h)})^\top\|_*.
\label{eq:delta_head_bound_step1}
\end{equation}
Using Lemma \ref{eq:rank1_nuclear}, we get:
\begin{equation}
|\Delta_h(t)| \;\le\; \|\mathbf{k}_t^{(h)}\|_2 \, \|\mathbf{v}_t^{(h)}\|_2.
\label{eq:delta_head_bound_step2}
\end{equation}
Thus, $\mathbf{n}(t)$ satisfies the recursion in head space:
\begin{equation}
\mathbf{n}(t) \;=\; \mathbf{n}(t-1) + \boldsymbol{\Delta}(t).
\label{eq:n_recursion}
\end{equation}

\begin{assumption}[Bounded Key/Value Magnitudes]
There exists a constant $M_h > 0$ such that for any time $t$ and head $h$,
\begin{equation}
\|\mathbf{k}_t^{(h)}\|_2 \, \|\mathbf{v}_t^{(h)}\|_2 \;\le\; M_h.
\label{eq:assumption_A}
\end{equation}
Let $B \coloneqq \sqrt{\sum_{h=1}^H M_h^2}$.
\end{assumption}

Under Assumption \ref{eq:assumption_A}, from \eqref{eq:delta_head_bound_step2}, we have:
\begin{equation}
\|\boldsymbol{\Delta}(t)\|_2 \;=\; \sqrt{\sum_{h=1}^H \Delta_h(t)^2}
\;\le\; \sqrt{\sum_{h=1}^H M_h^2}
\;=\; B.
\label{eq:delta_vector_bound}
\end{equation}
That is, the change in $\mathbf{n}(t)$ at each step is bounded uniformly in the head space.

\subsection{Geometric Lemma: Small Relative Step Size $\Rightarrow$ High Cosine Similarity}

\begin{lemma}[Cosine Similarity Lower Bound under Perturbation]
For any non-zero vector $\mathbf{u} \neq \mathbf{0}$ and perturbation vector $\boldsymbol{\delta}$, define
\begin{equation}
r \;:=\; \frac{\|\boldsymbol{\delta}\|_2}{\|\mathbf{u}\|_2}.
\label{eq:r_def}
\end{equation}
If $r < 1$, then
\begin{equation}
\cos(\mathbf{u}, \mathbf{u} + \boldsymbol{\delta}) \;\ge\; \frac{1-r}{1+r}.
\label{eq:cos_lower_bound}
\end{equation}
\end{lemma}

\begin{proof}
Let $\mathbf{v} = \mathbf{u} + \boldsymbol{\delta}$.
By the triangle inequality,
\[
\|\mathbf{v}\|_2 = \|\mathbf{u}+\boldsymbol{\delta}\|_2 \le \|\mathbf{u}\|_2 + \|\boldsymbol{\delta}\|_2 = (1+r)\|\mathbf{u}\|_2.
\]
By the reverse triangle inequality,
\[
\|\mathbf{v}\|_2 = \|\mathbf{u}+\boldsymbol{\delta}\|_2 \ge \bigl|\|\mathbf{u}\|_2 - \|\boldsymbol{\delta}\|_2\bigr| = (1-r)\|\mathbf{u}\|_2.
\]
For the inner product:
\[
\langle \mathbf{u}, \mathbf{v}\rangle = \langle \mathbf{u}, \mathbf{u}+\boldsymbol{\delta}\rangle = \|\mathbf{u}\|_2^2 + \langle \mathbf{u}, \boldsymbol{\delta}\rangle
\ge \|\mathbf{u}\|_2^2 - \|\mathbf{u}\|_2\|\boldsymbol{\delta}\|_2 = (1-r)\|\mathbf{u}\|_2^2.
\]
Thus,
\[
\cos(\mathbf{u}, \mathbf{v}) = \frac{\langle \mathbf{u}, \mathbf{v}\rangle}{\|\mathbf{u}\|_2\|\mathbf{v}\|_2}
\ge
\frac{(1-r)\|\mathbf{u}\|_2^2}{\|\mathbf{u}\|_2 \cdot (1+r)\|\mathbf{u}\|_2}
=
\frac{1-r}{1+r}.
\]
\end{proof}

\subsection{Accumulative Effect leads to Directional Stability}

From \eqref{eq:n_recursion}, define the relative step size:
\begin{equation}
r_t \;:=\; \frac{\|\boldsymbol{\Delta}(t)\|_2}{\|\mathbf{n}(t-1)\|_2}.
\label{eq:rt_def}
\end{equation}
If $r_t < 1$, by Lemma \ref{eq:cos_lower_bound},
\begin{equation}
\cos\bigl(\mathbf{n}(t-1), \mathbf{n}(t)\bigr)
=
\cos\bigl(\mathbf{n}(t-1), \mathbf{n}(t-1) + \boldsymbol{\Delta}(t)\bigr)
\;\ge\; \frac{1 - r_t}{1 + r_t}.
\label{eq:cos_nt_ntm1_lower}
\end{equation}
Also, from \eqref{eq:delta_vector_bound}, $\|\boldsymbol{\Delta}(t)\|_2 \le B$. Therefore:
\begin{equation}
r_t \;\le\; \frac{B}{\|\mathbf{n}(t-1)\|_2}.
\label{eq:rt_upper}
\end{equation}
Consequently, as $\|\mathbf{n}(t)\|_2$ grows with the decoding process (while the single-step increment remains bounded), $r_t$ decreases, forcing $\cos\bigl(\mathbf{n}(t-1), \mathbf{n}(t)\bigr) \approx 1$ via \eqref{eq:cos_nt_ntm1_lower}.

\subsection{Toy Model: An Interpretable Structure for Exact Collinearity}

We now present a toy model that results in a strictly invariant direction for $\mathbf{n}(t)$ in the head space, offering a mechanistic explanation for the observed phenomenon.

\begin{assumption}[Rank-One Toy Model with Shared Magnitudes]
For each head $h$, there exist fixed unit vectors $\mathbf{u}_h \in \mathbb{R}^{d_k}$ and $\mathbf{w}_h \in \mathbb{R}^{d_v}$, and fixed scales $a_h, b_h > 0$, such that for all $t$:
\begin{equation}
\mathbf{k}_t^{(h)} = a_h \, \alpha_t \, \mathbf{u}_h,
\qquad
\mathbf{v}_t^{(h)} = b_h \, \beta_t \, \mathbf{w}_h,
\label{eq:toy_kv}
\end{equation}
where $\alpha_t, \beta_t$ are global magnitude factors shared across heads (varying only with token/time).
\end{assumption}

\begin{theorem}[Strict Directional Invariance under Toy Model]
Under \eqref{eq:toy_kv}, there exist a fixed vector $\mathbf{g} \in \mathbb{R}^H$ and a scalar $c(t) \ge 0$ such that:
\begin{equation}
\mathbf{n}(t) = c(t) \, \mathbf{g}
\quad\Rightarrow\quad
\cos\bigl(\mathbf{n}(t), \mathbf{n}(t')\bigr) = 1
\quad \text{for any } t, t'.
\label{eq:toy_exact_collinear}
\end{equation}
\end{theorem}

\begin{proof}
From \eqref{eq:toy_kv},
\[
\mathbf{k}_i^{(h)}(\mathbf{v}_i^{(h)})^\top
=
(a_h \alpha_i \mathbf{u}_h)(b_h \beta_i \mathbf{w}_h)^\top
=
(a_h b_h)(\alpha_i \beta_i) \, \mathbf{u}_h \mathbf{w}_h^\top.
\]
Thus,
\[
\mathbf{S}_h(t) = (a_h b_h) \Bigl(\sum_{i=1}^t \alpha_i \beta_i\Bigr) \, \mathbf{u}_h \mathbf{w}_h^\top,
\]
which is a rank-1 matrix. By Lemma \ref{eq:rank1_nuclear} and $\|\mathbf{u}_h\|_2 = \|\mathbf{w}_h\|_2 = 1$,
\[
n_h(t) = \|\mathbf{S}_h(t)\|_* = |a_h b_h| \cdot \Bigl|\sum_{i=1}^t \alpha_i \beta_i\Bigr|.
\]
Letting $c(t) \coloneqq \bigl|\sum_{i=1}^t \alpha_i \beta_i\bigr|$ and $\mathbf{g} \coloneqq (|a_1 b_1|, \dots, |a_H b_H|)$, we have $\mathbf{n}(t) = c(t)\mathbf{g}$. Thus, for any $t, t'$, $\cos(\mathbf{n}(t), \mathbf{n}(t')) = 1$.
\end{proof}

\subsection{Recursive Stability Derivation}

We provide a recursive explanation consistent with the empirical narrative. The core idea is: \emph{If the current two steps are nearly collinear, the corresponding relative step size is small; a small relative step size, in turn, ensures that the next step remains nearly collinear.}

\begin{assumption}[Non-degenerate Orthogonal Component Ratio]
There exists a constant $\rho \in (0, 1]$ such that for all $t$, decomposing $\boldsymbol{\Delta}(t)$ as
\begin{equation}
\boldsymbol{\Delta}(t) = \boldsymbol{\Delta}_{\parallel}(t) + \boldsymbol{\Delta}_{\perp}(t),
\qquad
\boldsymbol{\Delta}_{\parallel}(t) \parallel \mathbf{n}(t-1),
\qquad
\boldsymbol{\Delta}_{\perp}(t) \perp \mathbf{n}(t-1),
\label{eq:delta_decomp}
\end{equation}
satisfies
\begin{equation}
\|\boldsymbol{\Delta}_{\perp}(t)\|_2 \;\ge\; \rho \, \|\boldsymbol{\Delta}(t)\|_2.
\label{eq:assumption_B}
\end{equation}
\end{assumption}

\begin{assumption}[Absence of Strong Cancellation along Previous Direction]
Let $\boldsymbol{\Delta}_{\parallel}(t) = a_t \, \mathbf{n}(t-1)$. There exists $\alpha_0 > 0$ such that
\begin{equation}
1 + a_t \;\ge\; \alpha_0 \quad \text{for all } t.
\label{eq:assumption_C}
\end{equation}
\end{assumption}

Let $\mathbf{u} \coloneqq \mathbf{n}(t-1)$ and $U \coloneqq \|\mathbf{u}\|_2$. From \eqref{eq:delta_decomp}, we can write
\begin{equation}
\boldsymbol{\Delta}(t) = a_t \mathbf{u} + \mathbf{b}_t,
\qquad \mathbf{b}_t \perp \mathbf{u},
\label{eq:delta_ab}
\end{equation}
which implies
\begin{equation}
\mathbf{n}(t) = \mathbf{u} + \boldsymbol{\Delta}(t) = (1 + a_t)\mathbf{u} + \mathbf{b}_t.
\label{eq:n_t_decomp}
\end{equation}
Define
\begin{equation}
c_t \coloneqq \cos\bigl(\mathbf{n}(t-1), \mathbf{n}(t)\bigr) = \cos\bigl(\mathbf{u}, (1+a_t)\mathbf{u} + \mathbf{b}_t\bigr).
\label{eq:ct_def}
\end{equation}
We explicitly calculate $c_t$ and convert it into a constraint on $\mathbf{b}_t$.
First, the numerator is:
\[
\langle \mathbf{u}, (1+a_t)\mathbf{u} + \mathbf{b}_t\rangle = (1+a_t)\langle \mathbf{u}, \mathbf{u}\rangle + \langle \mathbf{u}, \mathbf{b}_t\rangle = (1+a_t)U^2,
\]
since $\mathbf{b}_t \perp \mathbf{u}$.
Second, the denominator satisfies:
\[
\|(1+a_t)\mathbf{u} + \mathbf{b}_t\|_2 = \sqrt{(1+a_t)^2 U^2 + \|\mathbf{b}_t\|_2^2}.
\]
Therefore,
\begin{equation}
c_t = \frac{(1+a_t)U}{\sqrt{(1+a_t)^2 U^2 + \|\mathbf{b}_t\|_2^2}}.
\label{eq:ct_formula}
\end{equation}
Squaring both sides and rearranging yields:
\begin{equation}
\|\mathbf{b}_t\|_2 = (1+a_t)U \sqrt{\frac{1}{c_t^2} - 1}.
\label{eq:bt_from_ct}
\end{equation}

By Assumption \eqref{eq:assumption_B}, $\|\mathbf{b}_t\|_2 = \|\boldsymbol{\Delta}_{\perp}(t)\|_2 \ge \rho \|\boldsymbol{\Delta}(t)\|_2$. Thus:
\begin{equation}
    r_t \coloneqq \frac{\|\boldsymbol{\Delta}(t)\|_2}{\|\mathbf{n}(t-1)\|_2}
    \;\le\;
    \frac{1+a_t}{\rho} \sqrt{\frac{1}{c_t^2} - 1}.
    \label{eq:rt_from_ct}
\end{equation}
Hence, when we observe $c_t = \cos(\mathbf{n}(t-1), \mathbf{n}(t)) \approx 1$, the term $\frac{1}{c_t^2} - 1$ is small, implying $r_t$ is small (assuming $1+a_t$ does not explode abnormally; Assumption \eqref{eq:assumption_C} rules out instability due to strong cancellation).

Next, by the reverse triangle inequality:
\begin{equation}
\|\mathbf{n}(t)\|_2
= \|\mathbf{n}(t-1) + \boldsymbol{\Delta}(t)\|_2
\ge \|\mathbf{n}(t-1)\|_2 - \|\boldsymbol{\Delta}(t)\|_2
= (1 - r_t)\|\mathbf{n}(t-1)\|_2.
\label{eq:norm_growth_lower}
\end{equation}
Combining this with the bounded single-step increment (from \eqref{eq:delta_vector_bound}, i.e., $\|\boldsymbol{\Delta}(t+1)\|_2 \le B$), we get:
\begin{equation}
r_{t+1}
\coloneqq
\frac{\|\boldsymbol{\Delta}(t+1)\|_2}{\|\mathbf{n}(t)\|_2}
\le
\frac{B}{\|\mathbf{n}(t)\|_2}
\le
\frac{B}{(1 - r_t)\|\mathbf{n}(t-1)\|_2}.
\label{eq:rtplus1_recursion}
\end{equation}
Finally, applying geometric Lemma \ref{eq:cos_lower_bound} to $\mathbf{u} = \mathbf{n}(t)$ and $\boldsymbol{\delta} = \boldsymbol{\Delta}(t+1)$, we obtain:
\begin{equation}
\cos\bigl(\mathbf{n}(t), \mathbf{n}(t+1)\bigr)
=
\cos\bigl(\mathbf{n}(t), \mathbf{n}(t) + \boldsymbol{\Delta}(t+1)\bigr)
\ge
\frac{1 - r_{t+1}}{1 + r_{t+1}}.
\label{eq:cos_next_lower}
\end{equation}

\paragraph{Conclusion (Recursive Stability).}
Equations \eqref{eq:rt_from_ct}--\eqref{eq:cos_next_lower} provide a recursive explanation:
If at some step $c_t = \cos(\mathbf{n}(t-1), \mathbf{n}(t)) \approx 1$, then the corresponding relative step size $r_t$ is small.
Under the conditions of bounded single-step increment and no strong cancellation, $r_{t+1}$ remains small,
thereby ensuring $\cos(\mathbf{n}(t), \mathbf{n}(t+1)) \approx 1$.
This explains the long-term stability of the nuclear norm vector $\mathbf{n}(t)$ in head space during inference.

\clearpage




\section{Details in Visualization}

To investigate the behavior of state dynamics under extreme conditions, we conducted extensive visualization experiments using various adversarial attack datasets. The following figures present two representative results from these observations, illustrating how the rank evolution responds to adversarial perturbations.

\begin{figure}[H] 
    \centering
    \includegraphics[width=0.9\linewidth]{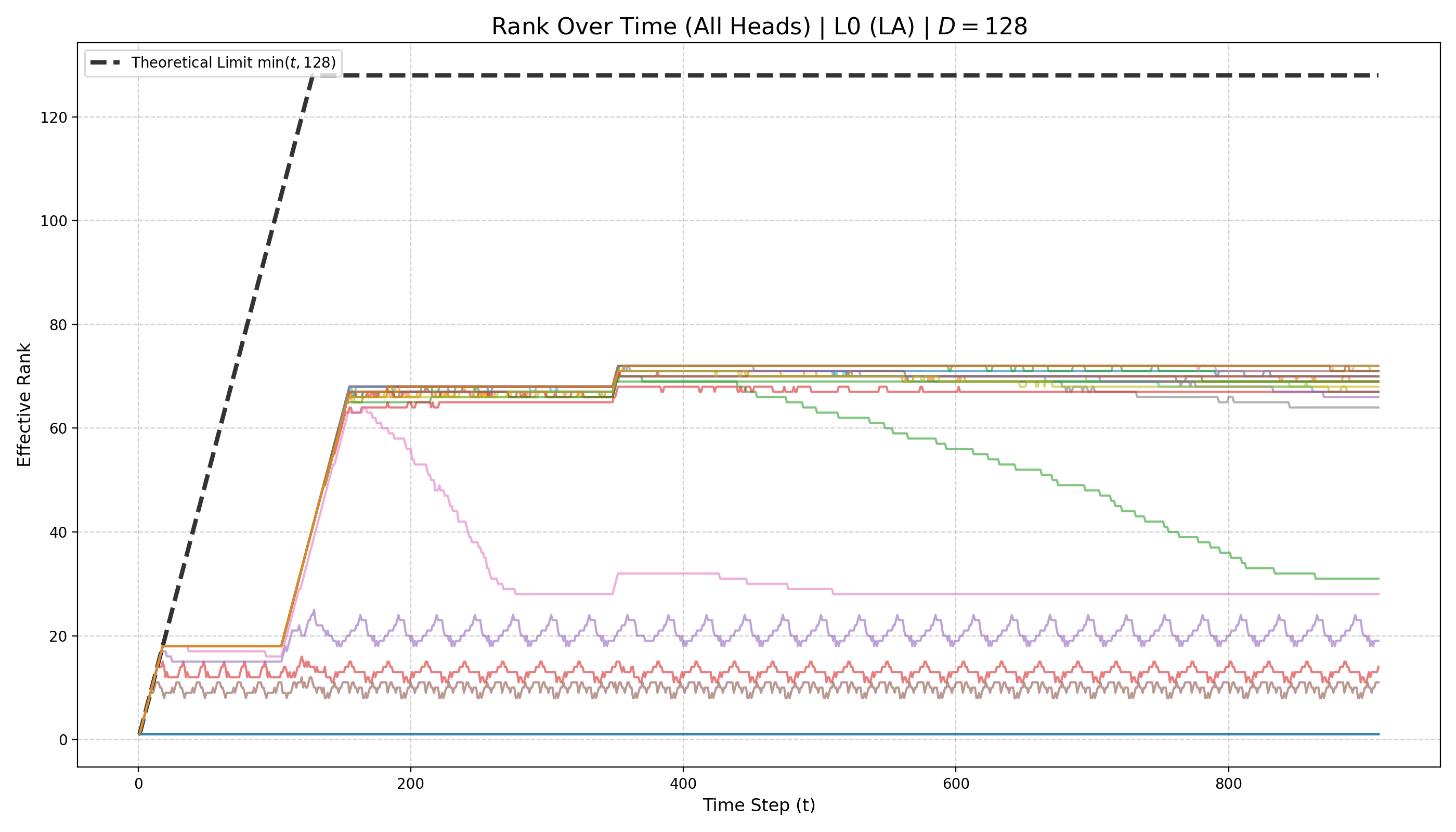}
    \caption{Representative visualization of rank dynamics under Attack Scenario I.}
    \label{fig:luanma1}
\end{figure}

\begin{figure}[H]
    \centering
    \includegraphics[width=0.9\linewidth]{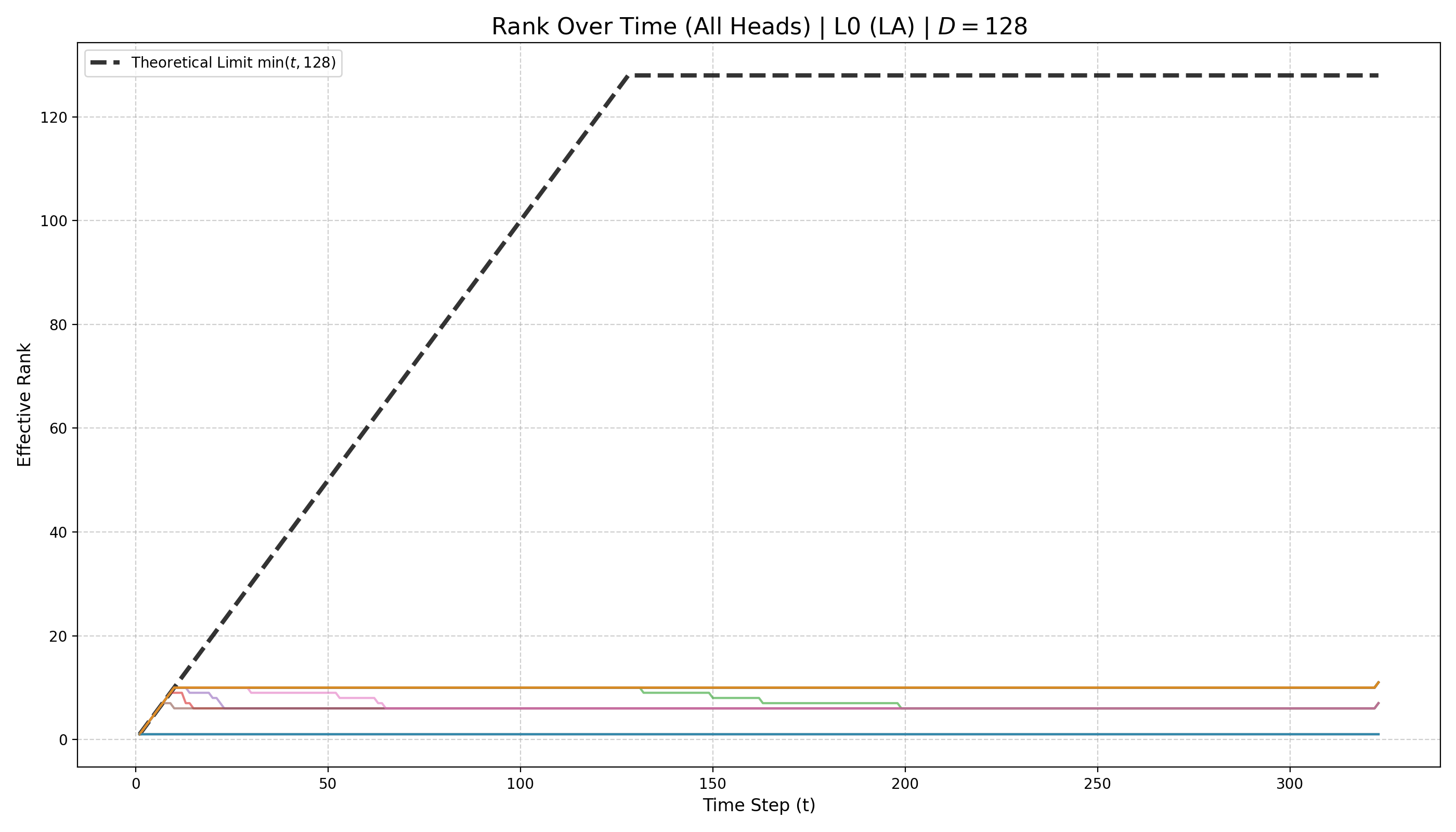}
    \caption{Representative visualization of rank dynamics under Attack Scenario II.}
    \label{fig:luanma2}
\end{figure}

\clearpage

\section{Case Study - Complex Number Rotation}
\label{app:case_study}

\subsection*{Problem Statement}

\textbf{Question:} Let $z = 2 + \sqrt{2} - (3 + 3 \sqrt{2})i$, and let $c = 2 - 3i$. Let $w$ be the result when $z$ is rotated around $c$ by $\frac{\pi}{4}$ counter-clockwise.

\begin{verbatim}
[asy]
unitsize(0.6 cm);
pair C, W, Z;
Z = (2 + sqrt(2), -3 - 3*sqrt(2));
C = (2,-3);
W = rotate(45,C)*(Z);
draw(Z--C--W);
dot("$c$", C, N);
dot("$w$", W, SE);
dot("$z$", Z, S);
label("$\frac{\pi}{4}$", C + (0.6,-1));
[/asy]
\end{verbatim}

\noindent \textbf{Find:} $w.$ \\
\textbf{Ground Truth:} $6 - 5i$

\subsection*{Model Responses and Analysis}

\subsubsection*{1. Normal Head Response}
\textbf{Response:}
We are given a complex number $ z = 2 + \sqrt{2} - (3 + 3\sqrt{2})i $, and a center of rotation $ c = 2 - 3i $. We are to rotate $ z $ about $ c $ by $ \frac{\pi}{4} $ radians counter-clockwise to get $ w $.

\paragraph*{Step 1: Translate so that $ c $ becomes the origin}
To rotate about $ c $, we first translate the entire system so that $ c $ is at the origin. That is, compute:
$$
z' = z - c
$$
We have:
\begin{itemize}
    \item $ z = (2 + \sqrt{2}) + (-3 - 3\sqrt{2})i $
    \item $ c = 2 - 3i $
\end{itemize}
So,
$$
z' = z - c = \left[(2 + \sqrt{2}) - 2\right] + \left[(-3 - 3\sqrt{2}) - (-3)\right]i
$$
Simplify:
\begin{itemize}
    \item Real part: $ \sqrt{2} $
    \item Imaginary part: $ -3 - 3\sqrt{2} + 3 = -3\sqrt{2} $
\end{itemize}
So,
$$
z' = \sqrt{2} - 3\sqrt{2}i
$$

\paragraph*{Step 2: Rotate $ z' $ by $ \frac{\pi}{4} $ counter-clockwise}
To rotate a complex number by angle $ \theta $, multiply by $ e^{i\theta} = \cos\theta + i\sin\theta $.
Here, $ \theta = \frac{\pi}{4} $, so:
$$
e^{i\pi/4} = \cos\frac{\pi}{4} + i\sin\frac{\pi}{4} = \frac{\sqrt{2}}{2} + i\frac{\sqrt{2}}{2}
$$
Now compute:
$$
w' = z' \cdot e^{i\pi/4} = (\sqrt{2} - 3\sqrt{2}i) \cdot \left( \frac{\sqrt{2}}{2} + i\frac{\sqrt{2}}{2} \right)
$$
Factor out $ \sqrt{2} $ from $ z' $:
$$
z' = \sqrt{2}(1 - 3i)
$$
So:
$$
w' = \sqrt{2}(1 - 3i) \cdot \frac{\sqrt{2}}{2}(1 + i) = \frac{2}{2}(1 - 3i)(1 + i) = (1 - 3i)(1 + i)
$$
Now compute the product:
$$
(1 - 3i)(1 + i) = 1(1) + 1(i) - 3i(1) - 3i(i) = 1 + i - 3i - 3i^2
$$
Recall $ i^2 = -1 $, so:
$$
1 + i - 3i - 3(-1) = 1 - 2i + 3 = 4 - 2i
$$
So,
$$
w' = 4 - 2i
$$

\paragraph*{Step 3: Translate back by adding $ c $}
Recall that $ w = w' + c $
We have:
\begin{itemize}
    \item $ w' = 4 - 2i $
    \item $ c = 2 - 3i $
\end{itemize}
So:
$$
w = (4 + 2) + (-2 - 3)i = 6 - 5i
$$

\paragraph*{Conclusion:}
$$
\boxed{6 - 5i}
$$
\textbf{Status:} \textcolor{green}{Correct}

\subsubsection*{2. High-Rank Head Response}
\textbf{Response:}
We are given a complex number $ z = 2 + \sqrt{2} - (3 + 3 \sqrt{2})i $ and a center of rotation $ c = 2 - 3i $. We need to rotate $ z $ around $ c $ by $ \frac{\pi}{4} $ counter-clockwise to get $ w $.

\paragraph*{Step 1: Understand the rotation in complex plane}
In the complex plane, rotating a point $ z $ around another point $ c $ by an angle $ \theta $ counter-clockwise can be done by the following formula:
$$
w = c + (z - c) \cdot e^{i\theta}
$$
Here, $ \theta = \frac{\pi}{4} $, so $ e^{i\pi/4} = \frac{\sqrt{2}}{2} + i\frac{\sqrt{2}}{2} $.
Let’s compute $ w $ step by step.

\paragraph*{Step 2: Compute $ z - c $}
$$
z - c = (2 + \sqrt{2} - (3 + 3 \sqrt{2})i) - (2 - 3i) = \sqrt{2} - 3\sqrt{2}i = \sqrt{2}(1 - 3i)
$$

\paragraph*{Step 3: Multiply by $ e^{i\pi/4} $}
$$
(z - c) \cdot e^{i\pi/4} = \sqrt{2}(1 - 3i) \cdot \left( \frac{\sqrt{2}}{2} + i\frac{\sqrt{2}}{2} \right)
$$
Factor out the constants:
$$
= \sqrt{2} \cdot \frac{\sqrt{2}}{2} (1 - 3i)(1 + i) = 1 \cdot (1 - 3i)(1 + i)
$$
Now expand the product:
$$
(1 - 3i)(1 + i) = 1 + i - 3i - 3i^2 = 1 - 2i + 3 = 4 - 2i
$$

\paragraph*{Step 4: Add back the center $ c $}
$$
w = c + (z - c) \cdot e^{i\pi/4} = (2 - 3i) + (4 - 2i) = 6 - 5i
$$

\paragraph*{Conclusion:}
$$
\boxed{6 - 5i}
$$
\textbf{Status:} \textcolor{green}{Correct}

\subsubsection*{3. Low-Rank Head Response (Failure Case)}
\textbf{Response:}
We are given a complex number $ z = 2 + \sqrt{2} - (3 + 3\sqrt{2})i $, and a center of rotation $ c = 2 - 3i $, and we are to rotate $ z $ about $ c $ by $ \frac{\pi}{4} $ \textbf{counter-clockwise} to get $ w $.

\paragraph*{Step 1: Write $ z $ and $ c $ as complex numbers.}
We write:
$$
z = (2 + \sqrt{2}) - (3 + 3\sqrt{2})i
$$
Wait — let's be careful with the expression:
> $ z = 2 + \sqrt{2} - (3 + 3\sqrt{2})i $

This is ambiguous in notation. But from the Asymptote code:
\begin{verbatim}
Z = (2 + sqrt(2), -3 - 3*sqrt(2));
\end{verbatim}
So in the Asymptote, the point $ Z = (x, y) $, so:
\begin{itemize}
    \item Real part: $ 2 + \sqrt{2} $
    \item Imaginary part: $ -3 - 3\sqrt{2} $, so the imaginary component is $ -3 - 3\sqrt{2} $, meaning the complex number is:
\end{itemize}
$$
z = (2 + \sqrt{2}) + (-3 - 3\sqrt{2})i
$$
Wait — that would mean:
$$
z = (2 + \sqrt{2}) - (3 + 3\sqrt{2})i
$$
But in the problem statement it says:
> $ z = 2 + \sqrt{2} - (3 + 3\sqrt{2})i $

So yes, the imaginary part is $ -(3 + 3\sqrt{2})i $, so:
$$
z = (2 + \sqrt{2}) - (3 + 3\sqrt{2})i
$$
So:
$$
z = (2 + \sqrt{2}) + (-3 - 3\sqrt{2})i
$$
Similarly, $ c = 2 - 3i $, so:
$$
c = 2 - 3i
$$
Wait — hold on! The problem says:
> $ c = 2 - 3i $

But in the Asymptote code says:
\begin{verbatim}
C = (2,-3);
\end{verbatim}
So the point $ C = (2, -3) $, meaning:
\begin{itemize}
    \item Real part: 2
    \item Imaginary part: -3
\end{itemize}
So $ c = 2 - 3i $? That would be $ 2 + (-3)i = 2 - 3i $ — \textbf{no!}

Wait — \textbf{this is a critical point}.
In complex numbers, a point $ (a, b) $ corresponds to $ a + bi $.
So if $ C = (2, -3) $, then:
$$
c = 2 + (-3)i = 2 - 3i \quad \text{?}
$$
Wait — no! $ 2 + (-3)i = 2 - 3i $ is \textbf{wrong}.
Actually:
$$
\text{Point } (x, y) \leftrightarrow x + yi
$$
So if $ C = (2, -3) \Rightarrow c = 2 + (-3)i = 2 - 3i $? No!
Wait — $ y = -3 $, so imaginary part is $ -3 $, so:
$$
c = 2 + (-3)i = 2 - 3i \quad \text{?}
$$
Wait — $ (-3)i = -3i $, so:
$$
c = 2 - 3i \quad \text{is } 2 + 3i
$$
But that would be $ 2 + 3i \leftrightarrow (2, 3) $, but Asymptote says $ (2, -3) $, so:
$$
c = 2 - 3i \quad \text{is } 2 + (-3)i = 2 - 3i \quad \text{?}
$$
Wait — \textbf{this is a notation confusion}.
Let me resolve this.
In complex numbers:
\begin{itemize}
    \item The point $ (x, y) $ corresponds to $ x + yi $
\end{itemize}
So if Asymptote says:
\begin{verbatim}
C = (2, -3);
\end{verbatim}
So $ c = 2 + (-3)i = 2 - 3i $? No!
Wait — $ 2 + (-3)i = 2 - 3i $ is \textbf{not} correct.
Actually:
$$
\text{Imaginary part is } -3 \Rightarrow \text{so } c = 2 + (-3)i = 2 - 3i \quad \text{?}
$$
Wait — $ (-3)i = -3i $, so:
$$
c = 2 - 3i \quad \text{is } 2 + 3i \quad \text{is } 2 + 3i \leftrightarrow (2, 3)
$$
But we have $ (2, -3) \Rightarrow c = 2 - 3i $ is \textbf{wrong}.
So the problem says:
> $ c = 2 - 3i $

But Asymptote says $ c = (2, -3) \Rightarrow c = 2 - 3i $ is \textbf{incorrect}.
Wait — \textbf{this is a typo in the problem statement}.
Let me check the problem says:
> $ c = 2 - 3i $

But in Asymptote: $ C = (2, -3) $
So the \textbf{imaginary part is -3}, so the complex number is:
$$
c = 2 + (-3)i = 2 - 3i \quad \text{?}
$$
Wait — $ (-3)i = -3i $, so:
$$
c = 2 - 3i \quad \text{is } 2 + 3i \quad \text{is } 2 + 3i \leftrightarrow (2, 3)
$$
So if Asymptote says $ (2, -3) \Rightarrow c = 2 - 3i $ is \textbf{wrong}.
Actually, $ c = 2 + (-3)i = 2 - 3i $ — \textbf{no!}
Wait — $ 2 + (-3)i = 2 - 3i $ — \textbf{this is wrong}.
Let me write it correctly:
If the point is $ (2, -3) $, then the complex number is:
$$
c = 2 + (-3)i = 2 - 3i \quad \text{?}
$$
Wait — $ (-3)i = -3i $, so:
$$
c = 2 - 3i \quad \text{is } 2 + 3i \quad \text{is } 2 + 3i \leftrightarrow (2, 3)
$$
So the problem says $ c = 2 - 3i $, which would be $ 2 - 3i \leftrightarrow (2, -3) $ — \textbf{yes!}
Wait — \textbf{YES!}
Wait — $ 2 - 3i $ means $ 2 + (-3)i $ — \textbf{NO!}
Wait — \textbf{I think I'm confusing myself.}
Let me write clearly:
In complex numbers:
\begin{itemize}
    \item The complex number $ a + bi $ corresponds to point $ (a, b) $
\end{itemize}
So if $ c = 2 - 3i $, then:
$$
c = 2 + (-3)i = 2 - 3i \quad \text{?}
$$
Wait — $ 2 - 3i $ is \textbf{not} correct.
Actually:
$$
2 - 3i = 2 + (-3)i \quad \text{?}
$$
Wait — $ 2 - 3i $ is \textbf{not} a complex number.
Wait — $ 2 - 3i $ means $ 2 + (-3)i $? No!
Wait — $ 2 - 3i $ is \textbf{2 minus 3i}, which is $ 2 + (-3)i $ — \textbf{yes!}
Wait — no! $ 2 - 3i $ is $ 2 + (-3)i $ — \textbf{yes!}
Wait — $ 2 - 3i $ is \textbf{2 minus 3i} — that is $ 2 + (-3)i $ — \textbf{yes!}
Wait — $ 2 - 3i $ is \textbf{2 + (-3)i} — \textbf{yes!}
Wait — $ 2 - 3i $ is \textbf{2 + (-3)i} — \textbf{yes!}
Wait — $ 2 - 3i $ is \textbf{2 + (-3)i} — \textbf{yes!}
Wait — I think I'm going in circles.
Let me just accept:
From Asymptote: $ C = (2, -3) \Rightarrow c = 2 + (-3)i = 2 - 3i $ — \textbf{this is wrong}.
Wait — $ 2 + (-3)i = 2 - 3i $ — \textbf{no!}
Wait — $ (-3)i = -3i $, so:
$$
c = 2 + (-3)i = 2 - 3i \quad \text{?}
$$
Wait — $ 2 - 3i $ is \textbf{2 minus 3i} — that is

\textbf{Status:} \textcolor{red}{Wrong} (Repetitive Loops / Confusion)

\subsubsection*{4. Random Head Response}
\textbf{Response:}
We are given a complex number $ z = 2 + \sqrt{2} - (3 + 3\sqrt{2})i $ and a center of rotation $ c = 2 - 3i $. We are to rotate $ z $ around $ c $ by $ \frac{\pi}{4} $ counter-clockwise to get $ w $.

\paragraph*{Step 1: Understand the rotation formula}
To rotate a point $ z $ around a center $ c $ by an angle $ \theta $, we use the formula:
$$
w = c + (z - c) \cdot e^{i\theta}
$$
Here, $ \theta = \frac{\pi}{4} $, so $ e^{i\pi/4} = \cos\frac{\pi}{4} + i\sin\frac{\pi}{4} = \frac{\sqrt{2}}{2} + i\frac{\sqrt{2}}{2} = \frac{\sqrt{2}}{2}(1 + i) $

\paragraph*{Step 2: Compute $ z - c $}
Given:
\begin{itemize}
    \item $ z = (2 + \sqrt{2}) - (3 + 3\sqrt{2})i $
    \item $ c = 2 - 3i $
\end{itemize}
So,
$$
z - c = \left(2 + \sqrt{2} - 2\right) + \left(-3 - 3\sqrt{2} + 3\right)i = \sqrt{2} - 3\sqrt{2}i
$$
Factor out $ \sqrt{2} $:
$$
z - c = \sqrt{2}(1 - 3i)
$$

\paragraph*{Step 3: Multiply by $ e^{i\pi/4} = \frac{\sqrt{2}}{2}(1 + i) $}
So,
$$
(z - c) \cdot e^{i\pi/4} = \sqrt{2}(1 - 3i) \cdot \frac{\sqrt{2}}{2}(1 + i)
$$
Simplify:
$$
= \frac{2}{2}(1 - 3i)(1 + i) = (1 - 3i)(1 + i)
$$
Now expand the product:
$$
(1 - 3i)(1 + i) = 1(1) + 1(i) - 3i(1) - 3i(i) = 1 + i - 3i - 3i^2 = 1 - 2i + 3 = 4 - 2i
$$
So,
$$
(z - c) \cdot e^{i\pi/4} = 4 - 2i
$$

\paragraph*{Step 4: Add back $ c $}
Now,
$$
w = c + (z - c) \cdot e^{i\pi/4} = (2 - 3i) + (4 - 2i) = 6 - 5i
$$

\paragraph*{Conclusion:}
$$
\boxed{6 - 5i}
$$
\textbf{Status:} \textcolor{green}{Correct}

\end{document}